\documentclass{article}

\usepackage{tikz}
\usepackage{graphicx}
\usepackage{paralist}

\usepackage{etoolbox}
\newtoggle{arxiv} \togglefalse{arxiv}
\newtoggle{generic} \togglefalse{generic}

\toggletrue{arxiv}

\toggletrue{generic}

\iftoggle{arxiv}{
\newcommand{\supplemental}{Appendix}
}{
\newcommand{\supplemental}{Supplemental Materials}
}

\iftoggle{arxiv}{%
    \iftoggle{generic}{%
    \usepackage[margin=1in]{geometry}
    }{
    \usepackage[preprint,nonatbib]{neurips_2022}
    }
}{%
  \usepackage[nonatbib]{neurips_2022}
}

\usepackage[utf8]{inputenc} % allow utf-8 input
\usepackage[T1]{fontenc}    % use 8-bit T1 fonts
\usepackage{hyperref}       % hyperlinks
\usepackage{url}            % simple URL typesetting
\usepackage{booktabs}       % professional-quality tables
\usepackage{amsfonts}       % blackboard math symbols
\usepackage{nicefrac}       % compact symbols for 1/2, etc.
\usepackage{microtype}      % microtypography
\usepackage{xcolor}         % colors
\usepackage{amsmath}
\usepackage{amsthm}
\usepackage{amssymb}
\usepackage{mathrsfs}
\usepackage[maxbibnames=9,maxcitenames=2,backend=biber]{biblatex}

\newtheorem{theorem}{Theorem}[section]

\newtheorem{lemma}[theorem]{Lemma}
\newtheorem{proposition}[theorem]{Proposition}
\newtheorem{corollary}[theorem]{Corollary}

\theoremstyle{definition}
\newtheorem{definition}[theorem]{Definition}
\newtheorem{example}[theorem]{Example}

\newcommand{\Ball}[3]{\mathtt{B}_{#2}(#3,#1)} % version 1
\newcommand{\KSR}[3]{\widehat{m}(#1\| #2,#3)} % kernel smoothing regression
\newcommand{\KSC}[3]{\widehat{u}(#1\| #2,#3)} % kernel smoothing classificaiton
\newcommand{\PC}[3]{\widehat{h}(#1\| #2,#3)} % partition classificaiton

\newcommand{\RHK}{K^{\mathcal{H}}} % Riemann-Hilbert kernel
\newcommand{\RPK}{K^{\mathtt{part}}_{\mathfrak{P},\alpha}} % Riemann-Hilbert kernel

\newcommand{\wrp}{{\mathsf{rp}}}
\newcommand{\rha}{{\mathsf{ha}}}
\newcommand{\WRPK}{K^{\mathsf{rp}}_M}
\newcommand{\RHAK}{K^{\mathsf{ha}}_{\mathbb{S}^d}}

\newcommand{\fn}{\mathtt{Fn}}
\newcommand{\vf}{\mathtt{Vf}}

\newcommand{\dist}{\mathtt{dist}}
\newcommand{\len}{\mathtt{len}}

\newcommand{\pd}[2]{\partial_{#1}|_{#2}}

\title{Consistent Interpolating Ensembles\\via the Manifold-Hilbert Kernel}

\iftoggle{arxiv}{%
  \author{
    Yutong Wang, Clayton Scott \\
    Electrical Engineering and Computer Science \\
    University of Michigan \\
    Ann Arbor, MI 48109\\
    \texttt{\{yutongw, clayscot\}@umich.edu} \\
  }
}{% neurips version

  \author{
    Yutong Wang \\
    University of Michigan\\
    \texttt{yutongw@umich.edu} \\
    \And
    Clayton D. Scott\\
    University of Michigan\\
    \texttt{clayscot@umich.edu}
  }

}

\begin{document}

\maketitle

\begin{abstract}
  Recent research in the theory of overparametrized learning has sought to establish generalization guarantees in the interpolating regime. Such results have been established for a few common classes of methods, but so far not for ensemble methods.
  We devise an ensemble classification method that simultaneously interpolates the training data, and is consistent for a broad class of data distributions.
  To this end, we define the \emph{manifold-Hilbert kernel} for data distributed on a Riemannian manifold.
  We prove that kernel smoothing regression and classification using the manifold-Hilbert kernel are weakly consistent in the setting of \textcite{devroye1998hilbert}.
  For the sphere, we show that the manifold-Hilbert kernel can be realized as a {weighted} random partition kernel, 
  which arises as an infinite ensemble of partition-based classifiers.
\end{abstract}

\section{Introduction}

Ensemble methods are among the most often applied learning algorithms, yet their theoretical properties have not been fully understood \parencite{biau2016random}.
Based on empirical evidence, \textcite{wyner2017explaining} conjectured that {interpolation of the training data} plays a key role in explaining the success of AdaBoost and random forests. However, while a few classes of learning methods have been analyzed in the interpolating regime \cite{belkin2019reconciling,bartlett2020benign}, ensembles have not.

Towards developing the theory of interpolating ensembles, we examine an ensemble classification method for data distributed on the sphere, and show that this classifier interpolates the training data and is consistent for a broad class of data distributions. To show this result, we develop two additional contributions that may be of independent interest. First, for data distributed on a Riemannian manifold $M$, we introduce the \emph{manifold-Hilbert kernel} $\RHK_M$, a manifold extension of the \emph{Hilbert kernel} \cite{shepard1968two}. Under the same setting as \textcite{devroye1998hilbert}, we prove that kernel smoothing regression  with $\RHK_M$ is weakly consistent while interpolating the training data. Consequently, the classifier obtained by taking the sign of the kernel smoothing estimate has zero training error and is consistent.

Second, we introduce a class of kernels called weighted random partition kernels. These are kernels that can be realized as an infinite, weighted ensemble of partition-based histogram classifiers. Our main result is established by showing that when $M = \mathbb{S}^d$, the $d$-dimensional sphere, the manifold-Hilbert kernel is a weighted random partition kernel. In particular, we show that on the sphere, the manifold-Hilbert kernel is a weighted ensemble based on random hyperplane arrangements.
This implies that the kernel smoothing classifier is a consistent, interpolating ensemble on $\mathbb{S}^d$. To our knowledge, this is the first demonstration of an interpolating ensemble method that is consistent for a broad class of distributions in arbitrary dimensions.

\subsection{Problem statement}
Consider the problem of binary classification on a Riemannian manifold $M$.
Let $(X,Y)$ be random variables jointly distributed on $M \times \{\pm 1\}$. Let $D^n:=\{(X_i,Y_i)\}_{i=1}^n$ be the (random) training data consisting of $n$ i.i.d copies of $X,Y$.
A \emph{classifier}, i.e., a mapping from $D^n$ to a function  $\widehat{f}(\bullet \| D^n): M \to \{ \pm 1\}$, has
the \textbf{interpolating-consistent property} if,
when $X$ has a continuous distribution,
both of the following hold:
1) 
$
\widehat{f}(X_i \| D^n) = Y_i,\, \mbox{for all $i \in \{1,\dots, n\}$}
$,
and
2) 
\begin{equation}
  \label{equation:consistency}
    \Pr\{
    \widehat{f}(X\| D^n)  \ne Y\}
    % \overset{i.p.}{\to}
    \to
  % \inf_{\substack{f: M \to \{\pm 1\}\\\mathrm{measurable}}}
  \inf_{{f: M \to \{\pm 1\}\,\,\mathrm{measurable}}}
    \Pr\{
    f(X)  \ne Y\}
    \quad \mbox{in probability as $n \to \infty$.}
\end{equation}
Our goal is to find an interpolating-consistent ensemble of \emph{histogram classifiers}, to be defined below.

A \emph{partition} on $M$, denoted by $\mathcal{P}$, is a set of subsets of $M$ such that $P \cap P' = \emptyset$ for all $P, P' \in \mathcal{P}$ and 
$M = \bigcup_{P \in \mathcal{P}} P$.
Given $x \in M$, let $\mathcal{P}[x]$ denote the unique element $P \in \mathcal{P}$ such that $x \in P$.
The set of all partitions on a space $M$ is denoted $\mathtt{Part}(M)$.
The \emph{histogram classifier} with respect to $D^n$ over $\mathcal{P}$ is the sign of the function $\PC{\bullet}{D^n}{\mathcal{P}} : M \to \mathbb{R}$ given by
\begin{equation}
  \label{equation:histogram-classifier}
  \PC{x}{D^n}{\mathcal{P}} : = \sum_{i = 1}^n Y_i \cdot \mathbb{I}\{x \in \mathcal{P}[X_i]\},
\end{equation}
where $\mathbb{I}$ is the indicator function.
\begin{definition}
  \label{definition:RP}
  A \emph{weighted random partition} (WRP)  over $M$ is a 3-tuple $(\Theta, \mathfrak{P}, \alpha)$
  consisting of
% \begin{enumerate}
  % \item 
  (i)
\emph{parameter space of partitions}:
a set $\Theta$
where $\mathcal{P}_\theta \in \mathtt{Part}(M)$ for each $\theta \in \Theta$, 
% parametrizing partitions on $M$, i.e., 
% \item 
(ii)
\emph{random partitions}:
 a probability measure $\mathfrak{P}$ on $\Theta$, and 
% \item 
 (iii)
\emph{weights}:
a nonnegative function
$\alpha : \Theta \to \mathbb{R}_{\ge 0}$.
% \end{enumerate}
\end{definition}
\begin{example}[Regular partition of the $d$-cube]
Let $M = [0,1]^d$ and $\Theta = \{1,2\dots\} =: \mathbb{N}_+$.
For each $n \in  \mathbb{N}_+$, denote by $\mathcal{P}_n$ the regular partition of $M$ into $n^d$ $d$-cubes of side length $1/n$. 
For any probability mass function $\mathfrak{P}$  on $\mathbb{N}_+$ and weights  $\alpha : \mathbb{N}_+ \to \mathbb{R}_{\ge 0}$, the 3-tuple $(\Theta, \mathfrak{P}, \alpha)$ is a WRP.
\end{example}

Below, WRPs will be denoted with 2-letter names in the sans-serif font,
e.g., ``$\wrp$'' for a generic WRP,
and ``$\rha$'' for the weighted hyperplane arrangement random partition (Definition~\ref{definition:random-HA-partition}).
The \emph{weighted random partition kernel} associated to 
$\wrp = (\Theta, \mathfrak{P}, \alpha)$
is
defined as
  \begin{equation}
    \label{equation:RPK}
    \WRPK :  M \times M \to \mathbb{R}_{\ge 0} \cup \{ \infty\},
    \quad
  \WRPK(x,z) := \mathbb{E}_{\theta \sim \mathfrak{P}} [
   \alpha(\theta) \mathbb{I}\{ x \in \mathcal{P}_\theta[z]\}
  ].
\end{equation} 
When $\alpha \equiv 1$, we recover the notion of unweighted random partition kernel introduced in \cite{davies2014random}.
Note that the kernel is symmetric since $\mathbb{I}\{x \in \mathcal{P}_\theta[z]\} = \mathbb{I}\{z \in \mathcal{P}_{\theta}[x]\}$. If $\WRPK < \infty$, then $\WRPK$ is a positive definite (PD) kernel. When $\WRPK$ can evaluate to $\infty$, the definition of a PD kernel is not applicable since the positive definite property is defined only for to kernels taking finite values \cite{berlinet2011reproducing}.

Let $\mathtt{sgn} : \mathbb{R} \cup \{\pm \infty\} \to \{\pm 1\}$ be the sign function.
For a WRP, define
the weighted infinite-ensemble 
\begin{equation}
  \label{equation:definition-kernel-smoothing-classification}
  \KSC{x}{D^n}{\WRPK} := 
\sum_{i=1}^n Y_i \cdot \WRPK(x,X_i)
=
  \mathbb{E}_{\theta \sim \mathfrak{P}}[
  \alpha(\theta)
  \PC{x}{D^n}{\mathcal{P}_\theta}].
\end{equation}
Note that the equality on the right follows immediately from linearity of the expectation and the definition of 
  $\PC{\bullet}{D^n}{\mathcal{P}_\theta}]$ in 
  Equation~\eqref{equation:histogram-classifier}.

\textbf{Main problem}. 
Find a WRP such that $\mathtt{sgn} (\KSC{\bullet}{D^n}{\WRPK})$ has the interpolating-consistent property.

\subsection{Outline of approach and contributions}
In the regression setting, we have $(X,Y)$ jointly distributed on $M \times \mathbb{R}$. Let $m(x) := \mathbb{E}[Y | X =x ]$.
% Let $D^n:=\{(X_i,Y_i)\}_{i\in[n]}$ be the (random) training data consisting of $n$ i.i.d copies of $X,Y$.
Recall from \textcite[Equation (7)]{belkin2019does} the definition of the \emph{kernel smoothing estimator}
with a so-called \emph{singular}\footnote{The ``singular'' modifier refers to the fact that $K(x,x) = +\infty$ for all $x \in M$.} kernel $K : M \times M \to [0,+\infty]$:
\begin{equation}
  \label{equation:definition-kernel-smoothing-regression}
  \KSR{x}{D^n}{K}
  := 
  \begin{cases}
    Y_i &: \exists i \in [n] \mbox{ such that } x = X_i \\
  \frac{\sum_{i = 1}^n Y_i K(x,X_i)}{\sum_{j=1}^n K(x,X_j)}
  &:\sum_{j=1}^n K(x,X_j) > 0
  \\
    0&: \mbox{otherwise.}
  \end{cases}
  % \,
  % \mbox{ $$ otherwise.  }
\end{equation}
% \begin{equation}
%   \label{equation:definition-kernel-smoothing-regression}
%   \KSR{x}{D^n}{K}
%   :=  \tfrac{\sum_{i = 1}^n Y_i K(x,X_i)}{\sum_{j=1}^n K(x,X_j)}
%   \,
%   \mbox{
% if $\sum_{j=1}^n K(x,X_j) > 0$, and 
% $
% \KSR{x}{D^n}{K} := 0$ otherwise.
% }
% \end{equation}
We note that Equation~\eqref{equation:definition-kernel-smoothing-regression} is referred as the \emph{Nadaraya-Watson} estimate in \cite{belkin2019does}.
Now, we simply write $\widehat{m}_n(x)$ instead of $\KSR{x}{D^n}{K}$ when there is no ambiguity. Similarly, we write $\widehat{u}_n(x)$ instead of $\KSC{x}{D^n}{K}$ from earlier.
Note that 
$\mathtt{sgn}(\widehat{m}_n(x)) = \mathtt{sgn}(\widehat{u}_n(x))$ if $\sum_{j=1}^n K(x,X_j) >0$.

Observe that $\widehat{m}_n$ is interpolating by construction.
Let $\mu_X$ denote the marginal distribution of $X$.
The \emph{$L_1$-error} of $\widehat{m}_n$ in approximating $m$ is 
$
  J_n := \int_{M} |\widehat{m}_n(x) - m(x)| \mu_X(dx)
$.
For $M = \mathbb{R}^d$ and the \emph{Hilbert kernel} defined by $K^{\mathcal{H}}_{\mathbb{R}^d}(x,z) := \|x-z\|^{-d}$,
\textcite{devroye1998hilbert} proved  $L_1$-consistency for \emph{regression}: $J_n \to 0$ in probability
when {$Y$ is bounded and $X$ is continuously distributed.
% See Definition~\ref{definition:density}}.
% % Remarkably, $K$ does not have a bandwidth parameter.
% Furthermore,
% $\widehat{m}_{n}$ interpolates the training data since
% $\lim_{x \to X_i} \widehat{m}_{n}(x) = Y_i$.
% The result of \cite{devroye1998hilbert} does not apply when $X$ does not have a density\footnote{when $\Pr(X =x) > 0$ for some $x$}.

  \textbf{Our contributions.} 
Our primary contribution is to demonstrate an ensemble method with the consistent-interpolating property. Toward this end, in Section~\ref{section:RH-kernel}, we introduce the manifold-Hilbert kernel $\RHK_M$ on a Riemannian manifold $M$.
When show that when $M$ is complete, connected, and smooth, kernel smoothing regression with $\RHK_M$ has the same consistency guarantee
(Theorem~\ref{theorem:main-theorem})
as $\RHK_{\mathbb{R}^d}$ mentioned in the preceding paragraph.
% Like the (ordinary) Hilbert kernel, kernel smoothing estimation with $\RHK_M$ interpolates the training data.
In Section~\ref{section:hawrp}, we consider the case when $M = \mathbb{S}^d$, and show that the manifold-Hilbert kernel $\RHK_{\mathbb{S}^d}$ is a weighted random partition kernel 
(Proposition~\ref{theorem:spherical-hilbert-kernel}).

\textcite[Section 7]{devroye1998hilbert} observed that the $L_1$-consistency of $\widehat{m}_n$ for regression implies the consistency for classification of $\mathtt{sgn} \circ \widehat{u}_n$.
Furthermore, $\widehat{m}_n$ is interpolating for regression implies that $\mathtt{sgn} \circ \widehat{u}_n$ is interpolating for classification.
% While the argument is presented in the context of Euclidean data, it is straightforward to see that the argument holds in the more general manifold setting.
These observations together with our results demonstrate the existence of a weighted infinite-ensemble classifier with the interpolating-consistent property.

% \subsection{Connection to ensemble classification}
% \label{section:RPK-to-ensemble-method}
% In this section, we briefly consider the binary classification setting when the labels $Y_i \in \{ \pm 1\}$.
% Given a weighted random partition $(\Theta, \mathfrak{P},\alpha)$, define $\gamma_\theta$ for a randomly sampled $\theta \sim \mathfrak{P}$ by
% $
%   \gamma_\theta(x)
%   := \sum_i Y_i  \alpha(\theta) \mathbb{I} \{ x \in \mathcal{P}_\theta[X_i]\}.
% $
% Taking expectation over $\mathbb{E}_{\theta \sim \mathfrak{P}}$, we have 
% \[
%   \widehat{u}_n(x)
%   :=
% \mathbb{E}_{\theta \sim \mathfrak{P}}[
%   \gamma_\theta(x)
%   ]
%   = \sum_i Y_i
% \mathbb{E}_{\theta \sim \mathfrak{P}}[
%   \alpha(\theta) \mathbb{I} \{ x \in \mathcal{P}_\theta[X_i]\}
%   ]
%   = \sum_i Y_i
%   \RPK(x,X_i).
%   % =:
%   % u_n(x).
% \]

% We define $
%   \widehat{u}_{n}$ to contrast this unnormalized estimator versus the previously defined normalized estimator
% $
% \widehat{m}_n
%   % :=  \frac{\sum_{i = 1}^n Y_i \RPK(x,X_i)}{\sum_{j=1}^n \RPK(x,X_i)}
% $ from the introduction.
% For regression, the estimator $\widehat{u}_n$ is not suitable. However, for classification, we are only interested in the sign of the estimator.
% Since the kernel $\RPK$ is nonnegative, we have $\mathtt{sgn} \circ \widehat{m}_n = \mathtt{sgn} \circ \widehat{u}_n$.
% As discussed in \textcite[Section 7]{devroye1998hilbert}, 
%     $J_n := \int |\widehat m(x) - m(x)| f(x) d \lambda(x) \to 0$ in probability
%     implies that $\Pr\{\mathtt{sgn} \circ \widehat{u}_n(X) \ne Y\}$ converges (in probability) to the Bayes classification error.

\subsection{Related work}

\textbf{Kernel regression.}
Kernel smoothing regression, or simply kernel regression, is an interpolator when the kernel used is singular, a fact known to \textcite{shepard1968two} in \citeyear{shepard1968two}.
\textcite{devroye1998hilbert} showed that kernel regression with the Hilbert kernel is interpolating and weakly consistent for data with a density and bounded labels.
Using singular kernels with compact support,
\textcite{belkin2019does} showed that minimax optimality can be achieved  under additional distributional assumptions.

% Under linear separation assumption on the data as the feature dimension tends to infinity, generalization error of AdaBoost in the interpolation regime has been analyzed by \textcite{liang2020precise}.

\textbf{Random forests.}
\textcite{wyner2017explaining} proposed that interpolation may be a key mechanism for the success of random forests and gave a compelling intuitive rationale.
\textcite{belkin2019reconciling} studied empirically the double descent phenomenon in random forests by considering the generalization performance past the interpolation threshold.
The PERT variant of random forests, introduced by \textcite{cutler2001pert}, provably interpolates in 1-dimension.
\textcite{belkin2018overfitting} pose as an interesting question whether the result of \textcite{cutler2001pert} extends to higher dimension.
Many work have established consistency of  random forest and its variants under different settings \cite{breiman2004consistency,biau2008consistency,scornet2015consistency}. However, none of these work addressed interpolation.

\textbf{Boosting.} For classification under the noiseless setting (i.e., the Bayes error is zero), AdaBoost is interpolating and consistent (see \textcite[first paragraph of Chapter 12]{freund2012boosting}).
However, this setting is too restrictive and the result does not answer if consistency is possible when fitting the noise.
\textcite{bartlett2007adaboost} proved that AdaBoost with early stopping is universally consistent, however without the interpolation guarantee.
To the best of our knowledge, whether AdaBoost or any other variant of boosting can be interpolating and consistent remains open.

\textbf{Random partition kernels.}
\textcite{breiman2000ensemble,geurts2006extremely} studied infinite ensembles of simplified variants of random forest and connections to certain kernels.
\textcite{davies2014random} formalized this connection and coined the term \emph{random partition kernel}.
\textcite{scornet2016random} further developed the theory of random forest kernels and obtained upper bounds on the rate of convergence.
However, it is not clear if these variants of random forests are interpolating.

Previously defined (unweighted) random partition kernels are bounded, and thus cannot be singular.
On the other hand, the manifold-Hilbert kernel is always singular.
To bridge between ensemble methods and theory on interpolating kernel smoothing regression, we propose \emph{weighted} random partitions (Definition~\ref{definition:RP}), whose associated kernel (Equation~\ref{equation:RPK}) can be singular.

\textbf{Learning on Riemannian manifolds.} 
Strong consistency of a kernel-based classification method on manifolds has been established by \textcite{loubes2008kernel}. However, the result requires the kernel to be bounded and thus the method is not guaranteed to be interpolating.
% \textcite{jayasumana2015kernel,feragen2015geodesic} study the conditions under which positive definiteness holds for Riemannian manifold version of classical kernels, e.g., Gaussian and Laplace kernel.
See \textcite{feragen2016open} for a review of theoretical results regarding kernels on Riemannian manifolds.
% of the state-of-the-art of the theory and interesting open questions.

Beyond kernel methods, other classical methods for Euclidean data have been extended to Riemannian manifolds, e.g.,
regression \cite{thomas2013geodesic},
classification \cite{yao2020principal}, and 
dimensionality reduction and clustering \cite{zhang2004principal}\cite{mardia2022principal}.
To the best of our knowledge, no previous works have demonstrated an interpolating-consistent classifiers on manifolds other than $\mathbb{R}^d$.

In many applications, the data naturally belong to a Riemannian manifold.
Spherical data arise from a range of disciplines in natural sciences. See
     the influential textbook by \textcite[Ch.1\S4]{mardia2000directional}.
% Machine learning for data distributed on Riemannian manifolds have been a longstanding topic.
For applications of the Grassmanian manifold in computer vision, see \textcite{jayasumana2015kernel} and the references therein.
% \textcite{fang2021kernel} study kernels in hyperbolic spaces motivated by image and natural language processing tasks with hyperbolic data.
% There has been efforts to \cite{huang2018building} deep learning on Grassmannian.
Topological data analysis \cite{wasserman2018topological} presents another interesting setting of manifold-valued data in the form of \emph{persistence diagrams} \cite{anirudh2016riemannian,le2018persistence}.
% \textcite{le2018persistence} studies a kernel on \emph{persistence diagrams}.

% For an introduction to measure-theoretic probability on Riemannian manifolds, see 
% \cite{pennec2006intrinsic}.

% For example of Grassmannian-valued data, see \cite{mankovich2022flag}.

% \subsection{Notations}

\section{Background on Riemannian Manifolds}

We give an intuitive overview of the necessary concepts and results on Riemannian manifolds. A longer, more precise version of this overview is in the \supplemental~Section~\ref{section:basics-of-riemannian-manifolds}.

A smooth $d$-dimensional manifold $M$ is a topological space that is locally diffeomorphic\footnote{A diffeomorphism is a smooth bijection whose inverse is also smooth.} to open subsets of $\mathbb{R}^d$.
% A \emph{chart} is a tuple $(U, \varphi)$ consisting of an open subset $U \subseteq M$ and a diffeomorphism $\varphi : U \to \varphi(U) \subseteq \mathbb{R}^d$ where $\varphi(U)$ is open.
For simplicity, suppose that $M$ is embedded in $\mathbb{R}^N$ for some $N \ge d$, e.g., $\mathbb{S}^d\subseteq \mathbb{R}^{d+1}$.
Let $x \in M$ be a point.
The \emph{tangent space} at $x$, denoted $T_xM$, is the set of vectors that is tangent to $M$ at $x$.
Since linear combinations of tangent vectors are also tangent, the tangent space $T_xM$ is a vector space.
Tangent vectors can also be viewed as the time derivative of smooth curves. In particular, let $x \in M$. If $\epsilon > 0$ is an open set and $\gamma : (-\epsilon,\epsilon) \to M$ is a smooth curve such that $\gamma(0) = x$, then $\frac{d\gamma}{dt}(0) \in T_{x}M$.

% Now, consider a fixed chart $(U,\varphi)$ and $x \in U$. Let $e_i \in \mathbb{R}^d$ be the $i$-th elementary basis vector. Then there exists an $\epsilon >0$ such that $\varphi(x) + t e_i \in \varphi(U)$ for all $t \in (-\epsilon,\epsilon)$. Thus,  $\gamma(t) := \varphi^{-1}(\varphi(x) + t e_i)$ defines a smooth curve $(-\epsilon,\epsilon) \to M$. Let $\pd{i}{x}$ be the tangent vector $\frac{d\gamma}{dt}(0)$. The set $\{\pd{i}{x} \}_{i=1}^d$ forms a basis for $T_xM$ called the \emph{natural basis} at $x$ induced by the chart $(U,\varphi)$.

A \emph{Riemannian metric} on $M$ is a choice of inner product $\langle \cdot, \cdot \rangle_x$ on $T_xM$ for each $x$ such that $\langle \cdot, \cdot \rangle_x$ varies smoothly with $x$.
% for each $x \in M$ such that for any chart $(U,\varphi)$, the function $g_{ij}: U \to \mathbb{R}$ given by $x \mapsto \langle \pd{i}{x},\pd{j}{x}\rangle_x$ is smooth for all $i,j \in [d]$.
% Denote by $G(x)$ the $d\times d$ positive definite matrix $[g_{ij}(x)]_{ij}$.
% The function $G: U \to \mathbb{R}^{d\times d}$ is called the \emph{coordinate representation} of the Riemannian metric with respect to the chart $(U,\varphi)$.
Naturally, $\|z\|_x := \sqrt{\langle z, z \rangle_x}$ defines a norm on $T_xM$. The length of a 
piecewise smooth curve $\gamma : [a,b] \to M$ is defined by $\len(\gamma) := \int_a^b \|\dot{\gamma}(t)\|_{\gamma(t)} dt$.
Define $\dist_M(x,\xi) := \inf \{ \len(\gamma) : \gamma$ is a piecewise smooth curve from $x$ to $\xi\}$, which is a metric on $M$ in the sense of metric spaces
(see \textcite[Proposition 1.1]{sakai1996riemannian}).
For $x \in M$ and $r \in (0,\infty)$, the open metric ball centered at $x$ of radius $r$ is denoted  
$\Ball{M}{x}{r} := \{\xi \in M : \dist_{M}(x,\xi) < r\}$.

A curve $\gamma : [a,b] \to M$ is a geodesic if $\gamma$ is \emph{locally} distance minimizing and has constant speed, i.e., $\|\frac{d\gamma}{dt}(\tau)\|_{\gamma(\tau)}$ is constant. Now, suppose $x \in M$ and $v \in T_xM$ are such that there exists a geodesic $\gamma:[0,1] \to M$ where $\gamma(0) = x$ and $\frac{d\gamma}{dt}(0) = v$. Define $\exp_x(v) := \gamma(1)$, the element reached by traveling along $\gamma$ at time $=1$. See Figure~\ref{fig:exponential-map} for the case when $M = \mathbb{S}^2$.

\begin{figure}
    \centering
    \includegraphics[width=0.95\textwidth]{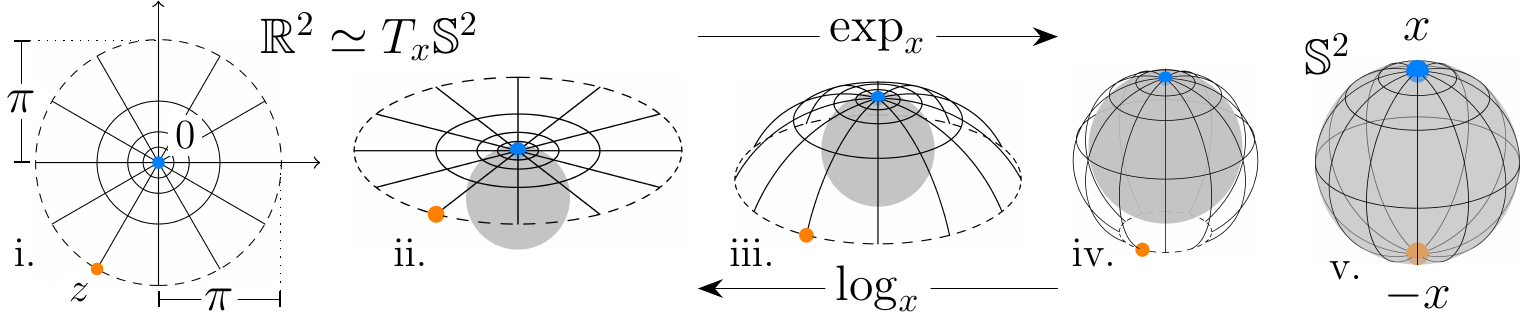}
    \caption{
    % \emph{Panel 1-5 (left to right)}: 
    An illustration of the exponential map $\exp_x$ for the manifold $M = \mathbb{S}^2$, where $x$ is the ``northpole'' (blue) and $-x$ the ``southpole'' (orange).
The logarithm map $\log_x$, discussed in Section~\ref{section:riemannian-logarithm}, is a right-inverse to $\exp_x$, i.e., $\exp_x \circ \log_x$ is the identity.
    \emph{Panel~i}. The tangent space $T_x \mathbb{S}^2$ visualized as $\mathbb{R}^2$. The dashed circle encloses a disc of radius $\pi$.
    \emph{Panel~ii}. The tangent space realized as the hyperplane tangent to sphere at $x$.
    \emph{Panel~iii-v}. Animation showing $\exp_x$ as a bijection from the open disc of radius $\pi$ to $\mathbb{S}^2 \setminus \{-x\}$. The entire dashed circle in Panel i is mapped to $-x$ the southpole.
    Thus, $\log_x$ maps the southpole $-x$ to a point $z$ on the dashed circle.
    }
    \label{fig:exponential-map}
\end{figure}

For a fixed $x \in M$, the above function $\exp_x$, the \emph{exponential map},  can be defined on an open subset of $T_xM$ containing the origin.
The 
Hopf-Rinow theorem (\cite[Ch.\ 8, Theorem 2.8]{do1992riemannian})
states that if $M$ is connected and complete with respect to the metric $\mathtt{dist}_M$, then  $\exp_x$ can be defined on all of $T_xM$.

\section{The Manifold-Hilbert kernel}
\label{section:RH-kernel}
Throughout the remainder of this work, we assume that $M$ is a complete, connected, and smooth Riemannian manifold of dimension $d$.

\begin{definition}
  \label{definition:manifold-Hilbert-kernel}
We define the \emph{manifold-Hilbert kernel} 
  $\RHK_{M} : M \times M \to [0,\infty]$ 
  for each $x,\xi \in M$
  by $\RHK_{M}(x,\xi) := \dist_M(x,\xi)^{-d}$ if $x \ne \xi$ and 
$\RHK_{M}(x,x) := \infty$ otherwise.
\end{definition}

Let $\lambda_M$ be the \emph{Riemann–Lebesgue volume measure} of $M$.
Integration with respect to this measure is denoted $\int_M f d\lambda_M$ for a function $f : M \to \mathbb{R}$.
For details of the construction of $\lambda_M$, see \textcite[Proposition 1.5]{amann2009integration}.
When $M =\mathbb{R}^d$, $\lambda_M$ is the ordinary Lebesgue measure and $\int_{\mathbb{R}^d} f d \lambda_{\mathbb{R}^d}$ is the ordinary Lebesgue integral. For this case, we simply write $\lambda$ instead of $\lambda_{\mathbb{R}^d}$.
% For a more expository references, see \textcite[\S 3]{pennec2006intrinsic} and \textcite[Ch.\ II \S 5]{sakai1996riemannian}.

We now state our first main result, a manifold theory extension of \textcite[Theorem 1]{devroye1998hilbert}.
\begin{theorem}
  \label{theorem:main-theorem}
  % Let $M$ be a complete connected Riemannian manifold, and $\lambda_M$ be the Riemannian measure.
  % Suppose that $X$ has a probability density function $f \in L^1(M, \mu)$ on $\mathbb{S}^d$ where $f$ is absolutely continuous w.r.t $\sigma^d$.
  Suppose that $X$ has a density $f_X$ with respect to $\lambda_M$ and that $Y$ is bounded.
  % Let $K := K_{RM}$ be the manifold-Hilbert kernel.
  Let $P_{Y|X}$ be a conditional distribution of $Y$ given $X$ and $m_{Y|X}$ be its conditional expectation.
  Let $\widehat{m}_n(x) := \KSR{x}{D^n}{\RHK_M}$.
  Then 
  \begin{compactenum}
    \item at almost all $x \in M$ with $f_X(x) > 0$, we have $\widehat{m}_{n}(x) \to m_{Y|X}(x)$ in probability,
    \item $J_n := \int_M |\widehat{m}_{n}(x) - m_{Y|X}(x)| f_X(x) d \lambda_M(x) \to 0$ in probability.
  \end{compactenum}
\end{theorem}
In words, the kernel smoothing regression estimate $\widehat{m}_{n}$ based on the manifold-Hilbert kernel is consistent and interpolates the training data, provided $X$ has a density and $Y$ is bounded. As a consequence, following the same logic as in \textcite{devroye1998hilbert}, the associated classifier $\mathtt{sgn} \circ \widehat{u}_n$ has the interpolating-consistent property. Before proving Theorem~\ref{theorem:main-theorem}, we first review key concepts in probability theory on Riemannian manifolds.
% When $M = \mathbb{S}^d$, applying Proposition~\ref{theorem:spherical-hilbert-kernel} with $q = -d$, we have that
% $
%   \RHAK(x,\xi) = \RHK_{\mathbb{S}^d}(x,\xi).
% $

\subsection{Probability on Riemannian manifolds}

% Let $M$ be a complete connected Riemannian manifold of dimension $d$.
Let $\mathcal{B}_M$ be the Borel $\sigma$-algebra of $M$, i.e., the smallest $\sigma$-algebra containing all open subsets of $M$.
% Let $M$ be a Riemannian manifold of dimension $d$.
We recall the definition of  $M$-valued random variables,
following \textcite[Definition 2]{pennec2006intrinsic}:
\begin{definition}
  \label{definition:probability-space}
  Let $(\Omega, \mathbb{P}, \mathcal{A})$ be a probability space with measure $\mathbb{P}$ and $\sigma$-algebra $\mathcal{A}$.
  A \emph{$M$-valued random variable} $X$ is a Borel-measurable function $\Omega \to M$, i.e., $X^{-1}(B) \in \mathcal{A}$ for all $B \in \mathcal{B}_M$.
\end{definition}
% Let $\mathcal{L}_M$ be the Lebesgue $\sigma$-algebra of $M$ (\cite[Proposition 1.2]{amann2009integration}).

\begin{definition}[Density]
  \label{definition:density}
  A random variable $X$ taking values in $M$ has a \emph{density} if there exists a nonnegative Borel-measurable function $f : M \to [0,\infty]$ such that for all Borel sets $B$ in $M$, we have 
  $
    \Pr(X \in B) = \int_B f d \lambda_M.
  $
  The function $f$ is said to be a \emph{probability density function} (PDF) of $X$.
\end{definition}
% Consider the kernel smoothing estimator using $\RPK(x,z) = \pi^{d} \angle(x,z)^{-d}$
% \[
%   m_n(x) :=  \frac{\sum_{i = 1}^n Y_i \angle(x,X_i)^{-d}}{\sum_{j=1}^n \angle(x,X_i)^{-d}}.
% \]

% \begin{definition}
%   [Density]
%   A random variable $X$ on $M$ has a density $f \in L^1(M,\mu)$ with respect to $\mu$ if 
%   $\Pr(X \in A) = \int_A f d\mu$ for all Borel measure sets $A$.
% \end{definition}

Next, we recall the definition of conditional distributions, following \textcite[Ch.\ 10 \S 2]{dudley2018real}:
\begin{definition}
  [Conditional distribution\footnote{
also known as \emph{disintegration measures} according to \textcite{chang1997conditioning}.
  }]
  \label{definition:conditional-distribution}
  Let $(X,Y)$ be a random variable jointly distributed on $M \times \mathbb{R}$.
  Let $P_X(\cdot)$ be the probability measure corresponding to the marginal distribution of $X$.
  A \emph{conditional distribution} for $Y$ given $X$ is a 
  collection of probability measures $P_{Y|X}(\cdot| x)$ on $\mathbb{R}$ indexed by $x \in M$ satisfying the following:
  % Let $x \in M$ be arbitrary, and  $A \subseteq \mathbb{R}$ and $B \subseteq M$ be Borel sets, then
  \begin{compactenum}
    % \item $P_{Y|X=x}$ is a probability measure on $\mathbb{R}$
    \item For all Borel sets $A \subseteq \mathbb{R}$,  the function $M \ni x \mapsto P_{Y|X}(A|x) \in [0,1]$
      is Borel-measurable.
    \item For all $A \subseteq \mathbb{R}$ and $B \subseteq M$ Borel sets, $\Pr(Y \in A, X \in B) = \int_B P_{Y|X}(A|x) P_X(dx)$.
  \end{compactenum}
  The \emph{conditional expectation}\footnote{More often, the conditional expectation is denoted $\mathbb{E}[Y|X=x]$. However, our notation is more convenient for function composition and compatible with that of \cite{devroye1998hilbert}.}
  is defined as $m_{Y|X}(x) := \int_{\mathbb{R}} y P_{Y|X}(dy|x)$. 
\end{definition}
The existence of a conditional probability for a joint distribution $(X,Y)$ is guaranteed by \textcite[Theorem 10.2.2]{dudley2018real}.
When $(X,Y)$ has a joint density $f_{XY}$ and marginal density $f_X$, the above definition gives the classical formula $P_{Y|X}(A|x) = \int_{A} f_{XY}(x,y)/f_X(x)  dy$ when $\infty>f_X(x) > 0$.
See the first example in \textcite[Ch.\ 10 \S 2]{dudley2018real}.

\subsection{Lebesgue points on manifolds}
% The kernel smoothing regression given the data $D_n:=\{(X_i,Y_i)\}_{i\in[n]}$ and kernel $K$ is the function
% \[
%   \widehat{m}_{D_n}^{K}(x) = \frac{\sum_{i\in[n]} Y_i\cdot K(x,X_i)}{\sum_{j\in[n]} K(x,X_i)}
%   \mbox{
% if $\sum_{j=1}^n k(x,X_i) > 0$ and $\widehat{m}_{D_n}^{K}(x)= 0$ otherwise.
% }
% \]
% When the data $D_n$ is clear from the context, we simply write $m_{n}^{K}$.

\textcite{devroye1998hilbert} proved Theorem~\ref{theorem:main-theorem} when $M = \mathbb{R}^d$ and, moreover, that
part 1 holds for the so-called \emph{Lebesgue points}, whose definition we now recall.

\begin{definition}
\label{definition:Lebesgue-point}
  Let $f: M \to \mathbb{R}$ be an absolutely integrable function and $x \in M$.
  We say that $x$ is a \emph{Lebesgue point} of $f$ if
  $
    f(x)=
    \lim_{r \to 0} \frac{1}{\lambda_M(\Ball{M}{x}{r})} \int_{\Ball{M}{x}{r}} f d \lambda_M
  $.
\end{definition}

% \begin{proposition}
%   [\cite{devroye1998hilbert}]
%   \label{proposition:main-theorem-euclidean-case}
%   Suppose that $Z$ has a probability density function $g \in L^1(\mathbb{R}^d, \lambda^d)$ on $\mathbb{R}^d$ where $f$ is absolutely continuous w.r.t $\lambda^d$. Suppose that $Y$ is bounded.
%   Let $\tilde{m}(z) := \mathbb{E}[Y| Z = z]$.
%   Suppose that $z \in \mathbb{R}^d$ is a Lebesgue point of $g(\cdot)$ and $g(\cdot) |\tilde{m}(\cdot)|$, then 
%   \[
%     \tilde m_n(z) := 
%     \frac{\sum_{i = 1}^n Y_i \|z - Z_i\|^{-d}}{\sum_{j=1}^n \|z- Z_j\|^{-d}}
%   \]
%   converges in probability to $\tilde m(z)$.

% \end{proposition}

% \section{Review of geometric measure theory}

For an integrable function, the following result states that almost all points are its Lebesgue points.
For the proof, see \textcite[Remark 2.4]{fukuoka2006mollifier}.
% The following theorem is well-known in the geometric measure theory literature.
\begin{theorem}
  [Lebesgue differentation]
  \label{theorem:LDT}
  % Suppose that $M$ is a Riemannian manifold, $d_{M}$ is the induced metric, and $\mu$ is a finite Borel measure.
  Let $f: M \to \mathbb{R}$ be an absolutely integrable function. Then there exists a set $A \subseteq M$ such that $\lambda_M(A) = 0$ and
  every $x \in M \setminus A$ is a Lebesgue point of $f$.
\end{theorem}
% The case when $M = \mathbb{R}^d$ was proven by Lebesgue.
Next, for the reader's convenience, we restate \textcite[Theorem 1]{devroye1998hilbert}, emphasizing the connection to Lebesgue points.
\begin{theorem}[\textcite{devroye1998hilbert}]
  \label{theorem:main-theorem-Devroye}
  Let $M = \mathbb{R}^d$ be the flat Euclidean space.
  Then Theorem~\ref{theorem:main-theorem} holds. 
  Moreover, Part 1 holds for all $x$ that is a Lebesgue point to both $f_X$ and $m_{Y|X} \cdot f_X$.
\end{theorem}
The above result will be used in our proof of Theorem~\ref{theorem:main-theorem} below.

  \section{Proof of Theorem~\ref{theorem:main-theorem}}

% We largely follow Chapter XII of .
%   We will first review several key notions and prove some novel results concerning Riemannian manifolds. 
  The focal point of the first subsection is  Lemma~\ref{lemma:Riemannian-logarithm} which shows the Borel measurability of extensions of the so-called Riemannian logarithm. The second subsection contains two key results regarding densities of $M$-valued random variables transformed by the Riemannian logarithm. The final subsection proves Theorem~\ref{theorem:main-theorem} leveraging results from the preceding two subsections.
% We will be using notations introduced in Section~\ref{section:basics-of-riemannian-manifolds}.
% Results cited from \cite{amann2009integration} are all from the Chapter XII therein.

\subsection{The Riemannian logarithm}
\label{section:riemannian-logarithm}
Throughout, $x$ is assumed to be an arbitrary point of $M$. 
Let $U_xM = \{v \in T_xM : \|v\|_x = 1\} \subseteq T_xM$ denote the set of unit tangent vectors.
Define a function $\tau_x: U_xM \to (0,\infty]$ as follows\footnote{Positivity of $\tau_x$ is asserted at \textcite[eq.\ (4.1)]{sakai1996riemannian}}:
\[
  \tau_x(u) := \sup\{ t > 0 : t = \dist_M(x, \exp_x(tu))\}.
\]
The \emph{tangent cut locus} is the set $\tilde{C}_x \subseteq T_xM$ defined by
$
  \tilde{C}_x := \{ \tau_x(u) u: u \in U_xM,\, \tau_x(u) < \infty\}.
$
Note that it is possible for $\tau_x(u) = \infty$ for all $u \in U_xM$ in which case $\tilde{C}_x$ is empty.
The \emph{cut locus} is the set $C_x:=\exp_x(\tilde{C}_x) \subseteq M$.

The \emph{tangent interior set} is 
$\tilde{I}_x := \{ tu: 0 \le t < \tau_x(u), u \in U_xM\}$
and the \emph{interior set} is the set $ I_x:=\exp_x(\tilde{I}_x)$.
Finally, define $\tilde{D}_x := \tilde{I}_x \cup \tilde{C}_x$.
Note that for each $z = tu \in \tilde{I}_x$, we have 
\begin{equation}
  \label{equation:exponential-map-preserves-distance}
  \|z\|_x = t = \dist_M(x,\exp_x(tu)) = \dist_M(x,\exp_x(z)).
\end{equation}
 % Let $C_x \subseteq M$ be the \emph{cut locus} of $p$.
% Define $\mathcal{I}_x = M \setminus C_x$.
% Also define $\tilde{C}_p$ and $\tilde{\mathcal{I}}_p$.
Consider the example where $M=\mathbb{S}^2$ as in Figure~\ref{fig:exponential-map}. Then $\tau_x(u) = \pi$ for all $u \in U_xM$. Thus, the tangent interior set $\tilde{I}_x = 
\Ball{\mathbb{R}^2}{0}{\pi}$, the open disc of radius $\pi$ centered at the origin.

When restricted to $\tilde{I}_x$, the exponential map $\exp_x|_{\tilde{I}_x}: \tilde{I}_x \to I_x$ is a diffeomorphism. Its functional inverse, denoted by $\log_x|_{I_x}$, is called the \emph{Riemannian Logarithm}
\cite{bendokat2020grassmann,zimmermann2017matrix}.
In previous works, $\log_x|_{I_x}$ is only defined from $I_x$ to $\tilde{I}_x$.
The next result shows that the domain of $\log_x|_{I_x} :  I_x \to \tilde{I}_x$ can be extended to $\log_x : M \to \tilde{D}_x$ while remaining Borel-measurable.
% \subsection{Riemannian logarithm}
\begin{lemma}
  \label{lemma:Riemannian-logarithm}
  For all $x \in M$, there exists a Borel measurable map $\log_x : M \to T_xM$ such that $\log_x(M) \subseteq \tilde{{D}}_x$ and $\exp_x \circ \log_x$ is the identity on $M$.
  Furthermore, for all $x,\xi \in M$, we have 
$\dist_M(x,\xi) = \| \log_x (\xi)\|_x$.
  % In particular, $\log_x$ is Lebesgue measurable.
\end{lemma}
\begin{proof}[Proof sketch]
The full proof of the lemma is provided in Section~\ref{appendix:proof-of-lemma:Riemannian-logarithm} of the \supplemental.
Below, we illustrate the idea of the proof using the example when $M=\mathbb{S}^2$ as in  Figure~\ref{fig:exponential-map}.

Let $x \in \mathbb{S}^2$ be the ``northpole'' (the blue point). The tangent cut locus $\tilde{C}_x$ is the dashed circle in the left panel of Figure~\ref{fig:exponential-map}. The exponential map $\exp_x$ is one-to-one on $\tilde{D}_x$ except on the dashed circle, which all gets mapped to $-x$, the ``southpole'' (the orange point). A consequence of the measurable selection theorem\footnote{Kuratowski–Ryll-Nardzewski measurable selection theorem (see \cite[Theorem 6.9.3]{bogachev2007measure})}
is that $\log_x$ can be extended to be a Borel-measurable right inverse of $\exp_x$ by selecting $z$ point on $\tilde{C}_x$ such that $\log_x(-x) = z$.
\end{proof}

\subsection{Random variable transforms}
% A \emph{chart} $(U,\varphi)$ of $M$ consists of an open subset $U \subseteq M$ and a diffeomorphism $\varphi : U \to \mathbb{R}^d$ where $\varphi(U)$ is open.
% Given a chart, there exists a smooth function\footnote{Note that $G_U$ depends not only on the set $U$ but also on $\varphi$. However, by writing $G_U$, we assume implicitly that $\varphi$ is given.}
% $G_{U} : U \to \mathbb{R}^{d\times d}$ called \emph{the Riemannian metric of $M$ on $U$} such that $G_{U}(x)$ is a positive definite matrix for each $x \in U$.

In the previous subsection, we showed that $\log_x : M \to T_xM$ is Borel-measurable. Now, recall that $T_xM$ is equipped with the inner product $\langle\cdot,\cdot \rangle_x$, i.e., the Riemannian metric. Below, for each $x \in M$ choose an orthonormal basis on $T_xM$ with respect to $\langle \cdot,\cdot \rangle$. Then $T_xM$ is isomorphic as an inner product space to $\mathbb{R}^d$ with the usual dot product.

Our first result of this subsection is a ``change-of-variables formula'' for computing the densities of $M$-valued random variables after the $\log_x$ transform. Recall that $\lambda_M$ is the Riemann-Lebesgue measure on $M$ and $\lambda$ is the ordinary Lebesgue measure on $\mathbb{R}^d = T_xM$.

\begin{proposition}
  \label{proposition:push-foward-distribution}
    Let $x \in M$ be fixed. There exists a Borel measurable function $\nu_x : M \to \mathbb{R}$ with the following properties:
    \begin{compactenum}[(i)]
      \item
      \label{item:logX}
      Let $X$ be a random variable on $M$ with density $f_X$
  and let $Z := \log_x(X)$. 
    Then $Z$ is a random variable on $T_xM$ with density $f_Z(z) := f_X( \exp_x(z)) \cdot  \nu_x ( \exp_x(z))$.
    \item 
    \label{item:mapping-lebesgue-points}
  Let $f: M \to \mathbb{R}$ be an absolutely integrable function such that $x$ is a Lebesgue point of $f$. Define $f: T_xM \to \mathbb{R}$ by $h(z) := f( \exp_x(z)) \cdot  \nu_x ( \exp_x(z))$.
  Then $0 \in T_xM$ is a Lebesgue point for $h$.
    \end{compactenum}
  
\end{proposition}
\begin{proof}[Proof sketch]
The full proof of the proposition is in  \supplemental~Section~\ref{appendix:proof-of-proposition:push-foward-distribution}.
The function $\nu_x$ is the Jacobian of the change-of-variables formula for integrating $\int_{\tilde{B}} 
    f_Z
    d\lambda$ where $\tilde{B} \subseteq T_xM$ is a Borel subset.
    See \supplemental~Lemma~\ref{lemma:change-of-variables}
for the exact definition of $\nu_x$. Part (\ref{item:logX}) is a simple consequence of this change-of-variables formula, which says that $\int_{\tilde{B}} 
    f_Z
    d\lambda = \int_{\exp_x(\tilde{B})} h d\lambda_M$.

For part (\ref{item:mapping-lebesgue-points}), the key observations are that (a) $\nu_x(\exp_x(0)) = \nu_x(0) = 1$ and (b) the volumes of $\Ball{M}{x}{r}$
and $\Ball{T_xM}{0}{r}$ are equal as $r \to 0$. More precisely,
$
\lim_{r \to 0} \frac{
\lambda_M(\Ball{M}{x}{r})
}
{
\lambda(\Ball{T_xM}{0}{r})
} 
= 1
$.
From these two observations, it is straightforward to directly verify Definition~\ref{definition:Lebesgue-point}.
\end{proof}

% \subsection{Conditional expectation under Riemannian logarithm}

\begin{proposition}
  \label{proposition:condition-expectation-transformation}  
Let $(X,Y)$ have a joint distribution on $M \times \mathbb{R}$ such that the marginal of $X$ has a density $f_X$ on $M$.
Let $P_{Y|X}(\cdot|\cdot)$ be a conditional distribution for $Y$ given $X$.
% Define $m_X(x) := \mathbb{E}[Y| X =x ]$.
Let $x \in M$.
Define $Z:=\log_x(X)$ and consider the joint distribution $(Z,Y)$ on $T_pM\times \mathbb{R}$.
Then $P_{Y|Z}(\cdot|\cdot) := P_{Y|X}(\cdot|\exp_x(\cdot))$ is a conditional distribution for $Y$ given $Z$.
% Consequently, $\mathbb{E}[Y|X=\exp_x(z)] = \mathbb{E}[Y|Z= z]$.
Consequently, $m_{Y|X}\circ \exp_x= m_{Y|Z}$.
\end{proposition}
\begin{proof}[Proof sketch]
The full proof of the Proposition is in  \supplemental~Section~\ref{appendix:proof-of-proposition:condition-expectation-transformation}. The idea is the same as in the proof of 
 Proposition~\ref{proposition:push-foward-distribution}, except that the probability density $f_Z$ is replaced by an appropriate conditional probability density.
\end{proof}

\subsection{
Finishing up the Proof of Theorem~\ref{theorem:main-theorem}
}

Fix $x \in M$ such that $x$ is a Lebesgue point of $f_X$ and $m_{Y|X} \cdot f_X$. 
Note that by Theorem~\ref{theorem:LDT}, almost all $x \in M$ has this property.
Next, let $Z = \log_x(X)$ and $f_Z$ be as in 
  Proposition~\ref{proposition:push-foward-distribution}-(\ref{item:logX}).
  Then
  \begin{compactenum}
  \item $f_Z=(f_X \circ \exp_x)  \cdot (\nu_x \circ \exp_x)$, and
\item 
$(m_{Y|X} \circ \exp_x) \cdot f_Z =(m_{Y|X} \circ \exp_x) \cdot (f_X\circ \exp_x)  \cdot (\nu_x \circ \exp_x)$.
  \end{compactenum}
  Now, proposition~\ref{proposition:push-foward-distribution}-(\ref{item:mapping-lebesgue-points}) implies that
  $0$ is a Lebesgue point of both $f_Z$ and $(m_{Y|X} \circ \exp_x) \cdot f_Z$.
  Furthermore, by Proposition~\ref{proposition:condition-expectation-transformation}, we have 
  $m_{Y|X} \circ \exp_x = m_{Y|Z}$.
  Thus, $0$ is a Lebesgue point of $f_Z$ and $m_{Y|Z}\cdot f_Z$.

% Let $X_1,\dots, X_n$ be i.i.d copies of a random variable $X$.
Now, let $D_n := \{(X_i,Y_i)\}_{i\in[n]}$.
Define $Z_i := \log_x(X_i)$, which are i.i.d copies of the random variable $Z := \log_x(X)$,
and let
$\tilde{D}_n := \{ (Z_i, Y_i)\}_{i\in[n]}$.
Then we have
\begin{align*}
  % \widehat{m}_{D_n}^{\RHK_M}(x) 
  \KSR{x}{D^n}{\RHK_M}
  &\overset{\mathrm{(a)}}{=}
  \frac{\sum_{i=1}^n Y_i \cdot \dist_M(x,X_i)^{-d}}{\sum_{j=1}^n \dist_M(x,X_j)^{-d}}
  \overset{\mathrm{(b)}}{=}
  \frac{\sum_{i=1}^n Y_i \cdot \|Z_i\|_x^{-d}}{\sum_{j=1}^n \|Z_j\|_x^{-d}}
  \\&\overset{\mathrm{(c)}}{=}
  \frac{\sum_{i=1}^n Y_i \cdot \dist_{\mathbb{R}^d}(0,Z_i)^{-d}}{\sum_{j=1}^n \dist_{\mathbb{R}^d}(0,Z_j)^{-d}}
%   \\&=
  \overset{\mathrm{(d)}}{=}
   \KSR{0}{\tilde{D}^n}{\RHK_{\mathbb{R}^d}}
\end{align*}
where equations marked by (a) and (d) follow from 
Equation~\eqref{equation:definition-kernel-smoothing-regression},
(b)  from  Lemma~\ref{lemma:Riemannian-logarithm}, and (c) from the fact that the inner product space $T_xM$ with $\langle \cdot, \cdot \rangle_x$ is isomorphic to $\mathbb{R}^d$ with the usual dot product.
  By Theorem~\ref{theorem:main-theorem-Devroye}, we have
$
  \KSR{0}{\tilde{D}^n}{\RHK_{\mathbb{R}^d}}
  \to
 m_{Y|Z}(0)
$
in probability. In other words, for all $\epsilon >0$,
\[
  \lim_{n\to\infty} 
\Pr\{
|
% m_{(Z,Y)^n}(0) 
  \KSR{0}{\tilde{D}^n}{\RHK_{\mathbb{R}^d}}
-m_{Y|Z}(0)| > \epsilon\}
=0.
\]
By Proposition~\ref{proposition:condition-expectation-transformation}, we have $m_{Y|Z}(0) = m_{Y|Z}(\exp_x(0)) = m_{Y|Z}(x)$.
Therefore, 
\[
  \Big\{
  |
  % m_{(Z,Y)^n}(0) 
  \KSR{0}{\tilde{D}^n}{\RHK_{\mathbb{R}^d}}
  -m_{Y|Z}(0)| > \epsilon
  \Big\}
  =
  \Big\{
    |
  % m_{(X,Y)^n}(x) 
  \KSR{x}{D^n}{\RHK_M}
  -m_{Y|X}(x)| > \epsilon
  \Big\}
\]
as events. Thus, $ 
  % m_{(X,Y)^n}(x) 
  \KSR{x}{D^n}{\RHK_M}
  \to m_{Y|X}(x)$ converges in probability, proving Theorem~\ref{theorem:main-theorem} part 1. As noted in \textcite[\S2]{devroye1998hilbert},  part 2 of Theorem~\ref{theorem:main-theorem} is an immediate consequence of part 1.

\section{Application to the $d$-Sphere}
\label{section:hawrp}}

The $d$-dimensional round sphere is $\mathbb{S}^d := \{ x \in \mathbb{R}^{d+1} : x^2_1 + \cdots + x^2_{d+1} = 1\}$.
Here, a \emph{round} sphere assumes that $\mathbb{S}^d$ has the \emph{arc-length metric}:
% \footnote{also known as the great-circle distance, orthodromic distance, or spherical distance.}:
\begin{equation}
  \label{equation:arclength-metric}
  \dist_{\mathbb{S}^d}(x,z) = \angle(x,z) = \cos^{-1} (x^\top z) \in [0,\pi].
\end{equation}

Let $\mathcal{S}$ be a set and $\sigma : M \to \mathcal{S}$ be a function. The \emph{partition induced} by $\sigma$ is defined by $\{ \sigma^{-1}(s) : s \in \mathtt{Range}(\sigma)\}$.
For example, when $M = \mathbb{S}^{d}$ and $W \in \mathbb{R}^{(d+1) \times h}$, then the function $\sigma_W : \mathbb{S}^{d} \to \{\pm 1\}^h$ defined by $\sigma_W(x)= \mathtt{sgn}(W^\top x)$ induces a hyperplane arrangement partition.

Let $\mathbb{N} = \{1,2,\dots \}$ and $\mathbb{N}_0 = \mathbb{N} \cup \{0\}$ denote the positive and non-negative integers.

\begin{definition}[Random hyperplane arrangement partition]
  \label{definition:random-HA-partition}
  Let $ d \in \mathbb{N}$ and $M = \mathbb{S}^{d}$. Let $q <0$ be a negative number, and let $H$ be a random variable with probability mass function $p_H : \mathbb{N}_0 \to [0,1]$ such that $p_H(h) > 0$ for all $h$.
Define the following weighted random partition $\rha := (\Theta, \mathfrak{P}, \alpha)$:
\begin{compactenum}
  % [label=\alph*).,itemsep=1.5pt,topsep=1.5pt]
  \item The parameter space $\Theta = \bigsqcup_{h = 0}^\infty \mathbb{R}^{(d+1) \times h}$ is the disjoint union of all $(d+1)\times h$ matrices.
    Element of $\Theta$ are matrices $\theta = W \in \mathbb{R}^{(d+1) \times h}$ where the number of columns $h \in \{0,1,2,\dots\}$ varies.
    By convention, if $h = 0$, the partition $\mathcal{P}_{\theta} = \mathcal{P}_{W}$ is the trivial partition $\{\mathbb{S}^d\}$. If $h > 0$, $\mathcal{P}_{W}$ is the partition induced by $x \mapsto \mathtt{sgn}(W^\top x)$.
    \item The probability $\mathfrak{P}$ is constructed by the procedure where we first sample $h \sim p_H(h)$, then sample the entries of $W \in \mathbb{R}^{d\times h}$ i.i.d\ according to   $\mathtt{Gaussian}(0,1)$.
    \item  For $\theta \in \Theta$, define
      $
    %   \label{equation:weight function}
          \alpha(\theta) := 
          \pi^{q}
  p_H(h)^{-1}
  (-1)^h \binom{q}{h}
  $, where
%   \quad \mbox{where} \quad
$\binom{q}{h} := 
% \frac
\tfrac{1}{h!}
\prod_{j=0}^{h-1} (q-j)
$.
\end{compactenum}
\end{definition}
Note that $(-1)^h \binom{q}{h} = 
\tfrac{1}{h!}
\prod_{j=0}^{h-1} (q-j)
> 0$ when $q < 0$.

% \begin{definition}
%   % [Random hyperplane partitions]
%   \label{definition:random-hyperplane-partitions}
%   Let $H$ be a random variable with discrete distribution $p_{H} : \mathbb{N}_0  \to [0,1]$ such that $p_{H}(h) >0$ for all $h \in \mathbb{N}_0$.
%   Define the distribution $\theta \sim \mathfrak{P}$ on $\mathbb{S}^d$ as follows: 
% \begin{enumerate}
%   \item Let $h$ be a realization of the random variable $H$
%   \item Sample a matrix $W \in \mathbb{R}^{{(d+1)} \times h}$ whose entries are i.i.d.\ standard Gaussian random variables.
%   \item Define $\mathcal{P}$ to be the partition induced by $x \mapsto \mathtt{sgn}(W^\top x)$.
% \end{enumerate}
% By convention, if $h = 0$, then $\mathcal{P} = \{\mathbb{S}^d\}$ is the trivial partition.
% \end{definition}

\begin{theorem}
  \label{theorem:spherical-hilbert-kernel}
  Let $\rha = (\Theta, \mathfrak{P}, \alpha)$ be as in Definition~\ref{definition:random-HA-partition}.
  Then
  \[
    \RHAK(x,z)
    =
    \begin{cases}
      \angle(x,z)^q &: \angle(x,z) \ne 0 \\
      + \infty &: \mbox{otherwise.}
    \end{cases}
  \]
  When $q = -d$, we have $\RHAK = \RHK_{\mathbb{S}^d}$ where the right hand side is the manifold-Hilbert kernel.
\end{theorem}

\begin{proof}[Proof of Theorem~\ref{theorem:spherical-hilbert-kernel}]
Before proceeding, we have the following useful lemma:
\begin{lemma}
  \label{lemma:ci-random-partition}
  Let $\wrp = (\Theta, \mathfrak{P},\alpha)$ be a WRP. Let $H$ be a random variable.
  % Define $\alpha(h) := \mathbb{E}[ \alpha(\theta) | H =h]$.
  Let $\theta \sim \mathfrak{P}$.
  Suppose that for all $x,z \in M$, the random variables $\alpha(\theta)$ and $\mathbb{I}\{x \in \mathcal{P}_\theta [z]\}$ are conditionally independent given $H$.
  Then we have
$
  \WRPK(x,z) = 
  \mathbb{E}_{H}\Big[
    \overline{\alpha}(H)
  \cdot
  \mathbb{E}_{\theta \sim \mathfrak{P}} [
  \mathbb{I}\{ x \in \mathcal{P}_\theta[z]\}
  | H
  ]
  \Big]
$
where 
$\overline{\alpha}(h):= \mathbb{E}_{\theta \in \mathfrak{P}} \left[\alpha(\theta) | H = h\right]$ for a realization $h$ of $H$.
\end{lemma}
The lemma follows immediately from the Definition of $\WRPK(x,z)$ in 
Equation~\ref{equation:RPK} and the conditional independence assumption. Now, we proceed with the proof of Theorem~\ref{theorem:spherical-hilbert-kernel}.

Let $\phi:=  \angle(x,z)/\pi$.
Let $H \sim p_H$ and $\theta \sim \mathfrak{P}$ be the random variables in Definition~\ref{definition:random-HA-partition}.
Note that by construction, the following condition is satisfied:
  for all $x,z \in M$, the random variables $\alpha(\theta)$ and $\mathbb{I}\{x \in \mathcal{P}_\theta [z]\}$ are conditionally independent given $H$.
  In fact,  $\alpha(\theta) =
          \pi^{q}
  p_H(h)^{-1}
  (-1)^h \binom{q}{h}
  $ is constant given $H =h $.
Hence, applying Lemma~\ref{lemma:ci-random-partition}, 
we have 
\begin{align*}
  &\RHAK(x,z) = 
  \mathbb{E}_{H}\Big[
    \overline{\alpha}(H)
  \cdot
  \mathbb{E}_{\theta \sim \mathfrak{P}} [
  \mathbb{I}\{ x \in \mathcal{P}_\theta[z]\}
  | H
  ]
  \Big]
      \\ &=
\sum_{h=0}^\infty
          \pi^{q}
  (-1)^h \binom{q}{h}
  \cdot
  \mathbb{E}_{\theta \sim \mathfrak{P}} [
  \mathbb{I}\{ x \in \mathcal{P}_\theta[z]\}
  | H =h ]
%   \\&=        
  =
  \sum_{h=0}^\infty
          \pi^{q}
  (-1)^h \binom{q}{h}
  \cdot
  \Pr\{ x \in \mathcal{P}_\theta[z] | H =h\}.
\end{align*}
Next, we claim that 
$
%   \label{equation:rhap-probability}
  \Pr\{ x \in \mathcal{P}_\theta[z] | H =h\} = (1 - \phi)^h.
$
When $h = 0$, $x \in \mathcal{P}_\theta[z]$ is always true since $\mathcal{P}_\theta = \{\mathbb{S}^{d}\}$ is the trivial partition.
In this case, we have $\Pr\{ x \in \mathcal{P}_\theta[z] | H =h\} = 1 = (1-\phi)^0$.
When $h >0$, we recall a result of
  \textcite{pinelis2019}:
  \begin{lemma}
    \label{lemma:angle-formula}
    Let $x,z \in \mathbb{S}^d$.
    Let $w \in \mathbb{R}^{d+1}$ be a random vector whose entries are 
    sampled 
i.i.d\ according to   $\mathtt{Gaussian}(0,1)$.
    Then $
    \Pr\{ \mathtt{sgn}(w^\top x) = \mathtt{sgn}(w^\top z) \}
    =1- (\angle(x,z)/\pi)$.
  \end{lemma}
  Let $W = [w_1,\dots, w_h]$ be as in Definition~\ref{definition:random-HA-partition} where $w_j$ denotes the $j$-th column of $W$.
  Then by construction, $w_j$ is distributed identically as $w$ in Lemma~\ref{lemma:angle-formula}.
  Furthermore, $w_j$ and $w_{j'}$ are independent for $j,j' \in [h]$ where $j \ne j'$.
  Thus, the claim follows from
% \begin{align*}
%   &\Pr\{ x \in \mathcal{P}_\theta[z] | H = h\}
%   \overset{\mathrm{(a)}}{=}
%   \Pr\{ \mathtt{sgn}(W^\top x) =\mathtt{sgn}(W^\top z) | H = h  \}
%   \\ &\overset{\mathrm{(b)}}{=}
%   \prod_{j = 1}^h
%   \Pr\{ \mathtt{sgn}(w_j^\top x) = \mathtt{sgn}(w_j^\top z)\}
% %   \\ &
%   \overset{\mathrm{(c)}}{=}
%   \prod_{j = 1}^h
% \left(1 - \phi\right)=
% \left(1 - \phi\right)^h.
% \end{align*}
\begin{align*}
  \Pr\{ x \in \mathcal{P}_\theta[z] | H = h\}
  & \overset{\mathrm{(a)}}{=}
  \Pr\{ \mathtt{sgn}(W^\top x) =\mathtt{sgn}(W^\top z) | H = h  \}
  \\ &\overset{\mathrm{(b)}}{=}
  \prod_{j = 1}^h
  \Pr\{ \mathtt{sgn}(w_j^\top x) = \mathtt{sgn}(w_j^\top z)\}
%   \\ &
  \overset{\mathrm{(c)}}{=}
  \prod_{j = 1}^h
\left(1 - \phi\right)=
\left(1 - \phi\right)^h.
\end{align*}
where equality (a) follows from Definition~\ref{definition:random-HA-partition}, (b) from
$W \in \mathbb{R}^{(d+1)\times h}$ having i.i.d\ standard Gaussian entries given $H = h$, and (c) from Lemma~\ref{lemma:angle-formula}.
Putting it all together, we have
\[
  \RPK(x,z) 
  =
  \sum_{h=0}^\infty 
\pi^q
  (-1)^h \binom{q}{h} (1- \phi)^h
  =
\pi^q
  \sum_{h=0}^\infty  \binom{q}{h} (\phi-1)^h
  =
  \angle(x,z)^q.
\]
For the last step, we used the fact that for all $q \in \mathbb{R}$ the binomial series
$
  (1+t)^{q}=\sum_{h=0}^{\infty }\binom{q}{h}t^h
$
converges absolutely for $|t| < 1$ (when $\phi \in (0,1]$) and diverges to $+\infty$ for $t = -1$ (when $\phi = 0$).
\end{proof}

\begin{corollary}
  Let $q := -d$ and $\RHAK$ be as in Theorem~\ref{theorem:spherical-hilbert-kernel}.
  The infinite-ensemble classifier $\mathtt{sgn} (\KSC{\bullet}{D^n}{\RHAK})$ (see Equation~\ref{equation:definition-kernel-smoothing-classification} for definition) has the interpolating-consistent property.
\end{corollary}
\begin{proof}
As observed in \textcite[Section 7]{devroye1998hilbert}, for an arbitrary kernel $K$, the $L_1$-consistency of $\KSR{\bullet}{D^n}{K}$ for regression implies the consistency for classification of $\mathtt{sgn}( \KSC{\bullet}{D^n}{K})$.
Furthermore, $\KSR{\bullet}{D^n}{K}$ is interpolating for regression implies that $\mathtt{sgn} (\KSC{\bullet}{D^n}{K})$ is interpolating for classification.
While the argument there is presented in the $\mathbb{R}^d$ case, the argument holds in the more general manifold case \emph{mutatis mutandis}.

Thus, by Theorem~\ref{theorem:main-theorem}, we have $\mathtt{sgn} (\KSC{\bullet}{D^n}{\RHK_{\mathbb{S}^d}})$ is consistent for classification, i.e., 
  Equation~\eqref{equation:consistency} holds.
  It is also interpolating since $\KSR{\bullet}{D^n}{K}$ is interpolating.
By Proposition~\ref{theorem:spherical-hilbert-kernel}, we have $\RHAK = \RHK_{\mathbb{S}^d}$. Thus 
$\mathtt{sgn} (\KSC{\bullet}{D^n}{\RHAK})$ is an ensemble method having the interpolating-consistent property.
\end{proof}

  \section{Discussion}
  \label{section:Discussion}
  We have shown that using the manifold-Hilbert kernel in kernel smoothing regression, also known as Nadaraya-Watson regression, results in a consistent estimator that interpolates the training data on a Riemannian manifold $M$.
  Furthermore, when $M = \mathbb{S}^d$ is the sphere, we showed that the manifold-Hilbert kernel is a weighted random partition kernel, where the random partitions are induced by random hyperplane arrangements.
  This demonstrates an ensemble method that has the interpolating-consistent property.

A limitation of this work is that the random hyperplane arrangement partition is data-independent.
Thus, the resulting ensemble method considered in this work are easier to analyze than popular ensemble methods used in practice.
  Nevertheless, we believe our work offers one theoretical basis towards understanding generalization in the interpolation regime of ensembles of histogram classifiers over data-dependent partitions, e.g., decision trees \`a la CART \cite{breiman1984classification}.

\iftoggle{arxiv}{% no checklist for arxiv
\section*{Acknowledgements}
The authors were supported in part by the National Science Foundation under awards 1838179 and 2008074.
}
{}
  \printbibliography

@book{mardia2000directional,
	author = {Mardia, Kanti V and Jupp, Peter E},
	publisher = {Wiley Online Library},
	title = {Directional statistics},
	volume = {2},
	year = {2000}}

@article{zhang2004principal,
	author = {Zhang, Zhenyue and Zha, Hongyuan},
	journal = {SIAM journal on scientific computing},
	number = {1},
	pages = {313--338},
	publisher = {SIAM},
	title = {Principal manifolds and nonlinear dimensionality reduction via tangent space alignment},
	volume = {26},
	year = {2004}}

@article{loubes2008kernel,
	author = {Loubes, Jean-Michel and Pelletier, Bruno},
	journal = {Statistics \& Decisions},
	number = {1},
	pages = {35--51},
	publisher = {Oldenbourg Wissenschaftsverlag GmbH},
	title = {A kernel-based classifier on a Riemannian manifold},
	volume = {26},
	year = {2008}}

@book{berlinet2011reproducing,
	author = {Berlinet, Alain and Thomas-Agnan, Christine},
	publisher = {Springer Science \& Business Media},
	title = {Reproducing kernel Hilbert spaces in probability and statistics},
	year = {2011}}

@article{thomas2013geodesic,
	author = {Thomas Fletcher, P},
	journal = {International journal of computer vision},
	number = {2},
	pages = {171--185},
	publisher = {Springer},
	title = {Geodesic regression and the theory of least squares on Riemannian manifolds},
	volume = {105},
	year = {2013}}

@article{mardia2022principal,
	author = {Mardia, Kanti V and Wiechers, Henrik and Eltzner, Benjamin and Huckemann, Stephan F},
	journal = {Journal of Multivariate Analysis},
	pages = {104862},
	publisher = {Elsevier},
	title = {Principal component analysis and clustering on manifolds},
	volume = {188},
	year = {2022}}

@article{yao2020principal,
	author = {Yao, Zhigang and Zhang, Zhenyue},
	journal = {Journal of the American Statistical Association},
	number = {531},
	pages = {1435--1448},
	publisher = {Taylor \& Francis},
	title = {Principal boundary on Riemannian manifolds},
	volume = {115},
	year = {2020}}

@article{wasserman2018topological,
	author = {Wasserman, Larry},
	journal = {Annual Review of Statistics and Its Application},
	pages = {501--532},
	publisher = {Annual Reviews},
	title = {Topological data analysis},
	volume = {5},
	year = {2018}}

@article{cutler2001pert,
	author = {Cutler, Adele and Zhao, Guohua},
	journal = {Computing Science and Statistics},
	pages = {490--497},
	title = {{PERT}-perfect random tree ensembles},
	volume = {33},
	year = {2001}}

@article{freund2012boosting,
	author = {Freund, Yoav and Schapire, Robert E},
	journal = {MIT Press},
	number = {6},
	pages = {7},
	title = {Boosting: Foundations and Algorithms},
	volume = {1},
	year = {2012}}

@inproceedings{anirudh2016riemannian,
	author = {Anirudh, Rushil and Venkataraman, Vinay and Natesan Ramamurthy, Karthikeyan and Turaga, Pavan},
	booktitle = {Proceedings of the IEEE conference on computer vision and pattern recognition workshops},
	pages = {68--76},
	title = {A {R}iemannian framework for statistical analysis of topological persistence diagrams},
	year = {2016}}

@article{le2018persistence,
	author = {Le, Tam and Yamada, Makoto},
	journal = {Advances in Neural Information Processing Systems},
	title = {Persistence {F}isher kernel: A {R}iemannian manifold kernel for persistence diagrams},
	volume = {31},
	year = {2018}}

@inproceedings{feragen2016open,
	author = {Feragen, Aasa and Hauberg, S{\o}ren},
	booktitle = {Conference on Learning Theory},
	organization = {PMLR},
	pages = {1647--1650},
	title = {Open Problem: Kernel methods on manifolds and metric spaces. What is the probability of a positive definite geodesic exponential kernel?},
	year = {2016}}

@inproceedings{shepard1968two,
	author = {Shepard, Donald},
	booktitle = {Proceedings of the 1968 23rd ACM national conference},
	pages = {517--524},
	title = {A two-dimensional interpolation function for irregularly-spaced data},
	year = {1968}}

@book{dudley2018real,
	author = {Dudley, Richard M},
	publisher = {CRC Press},
	title = {Real Analysis and Probability},
	year = {2018}}

@article{chang1997conditioning,
	author = {Chang, Joseph T and Pollard, David},
	journal = {Statistica Neerlandica},
	number = {3},
	pages = {287--317},
	publisher = {Wiley Online Library},
	title = {Conditioning as disintegration},
	volume = {51},
	year = {1997}}

@book{do1992riemannian,
	author = {Do Carmo, Manfredo Perdigao},
	publisher = {Springer},
	title = {Riemannian Geometry},
	volume = {6},
	year = {1992}}

@book{amann2009integration,
	author = {Amann, Herbert and Escher, Joachim},
	pages = {389--455},
	publisher = {Springer},
	title = {Analysis III},
	year = {2009}}

@article{zimmermann2017matrix,
	author = {Zimmermann, Ralf},
	journal = {SIAM Journal on Matrix Analysis and Applications},
	number = {2},
	pages = {322--342},
	publisher = {SIAM},
	title = {A matrix-algebraic algorithm for the Riemannian logarithm on the Stiefel manifold under the canonical metric},
	volume = {38},
	year = {2017}}

@article{jayasumana2015kernel,
	author = {Jayasumana, Sadeep and Hartley, Richard and Salzmann, Mathieu and Li, Hongdong and Harandi, Mehrtash},
	journal = {IEEE transactions on pattern analysis and machine intelligence},
	number = {12},
	pages = {2464--2477},
	publisher = {IEEE},
	title = {Kernel methods on Riemannian manifolds with Gaussian RBF kernels},
	volume = {37},
	year = {2015}}

@article{pennec2006intrinsic,
	author = {Pennec, Xavier},
	journal = {Journal of Mathematical Imaging and Vision},
	number = {1},
	pages = {127--154},
	publisher = {Springer},
	title = {Intrinsic statistics on {R}iemannian manifolds: {B}asic tools for geometric measurements},
	volume = {25},
	year = {2006}}

@article{bendokat2020grassmann,
	author = {Bendokat, Thomas and Zimmermann, Ralf and Absil, P-A},
	journal = {arXiv preprint arXiv:2011.13699},
	title = {A {G}rassmann manifold handbook: {B}asic geometry and computational aspects},
	year = {2020}}

@article{hebda1987parallel,
	author = {Hebda, James J},
	journal = {Transactions of the American Mathematical Society},
	number = {2},
	pages = {559--572},
	title = {Parallel translation of curvature along geodesics},
	volume = {299},
	year = {1987}}

@book{sakai1996riemannian,
	author = {Sakai, Takashi},
	publisher = {American Mathematical Society},
	title = {Riemannian Geometry},
	volume = {149},
	year = {1996}}

@book{bogachev2007measure,
	author = {Bogachev, Vladimir Igorevich and Ruas, Maria Aparecida Soares},
	publisher = {Springer},
	title = {Measure theory},
	volume = {1},
	year = {2007}}

@article{fukuoka2006mollifier,
	author = {Fukuoka, Ryuichi},
	journal = {arXiv preprint math/0608230},
	title = {Mollifier smoothing of tensor fields on differentiable manifolds and applications to Riemannian Geometry},
	year = {2006}}

@article{devroye1998hilbert,
	author = {Devroye, Luc and Gy\"{o}rfi, L\'{a}szl\'{o} and Krzy\.{z}ak, Adam},
	journal = {Journal of Multivariate Analysis},
	number = {2},
	pages = {209--227},
	publisher = {Elsevier},
	title = {The {H}ilbert kernel regression estimate},
	volume = {65},
	year = {1998}}

@article{scornet2015consistency,
	author = {Scornet, Erwan and Biau, G{\'e}rard and Vert, Jean-Philippe},
	journal = {The Annals of Statistics},
	number = {4},
	pages = {1716--1741},
	publisher = {Institute of Mathematical Statistics},
	title = {Consistency of random forests},
	volume = {43},
	year = {2015}}

@article{breiman2004consistency,
	author = {Breiman, Leo},
	publisher = {Citeseer},
	title = {Consistency for a simple model of random forests},
	year = {2004}}

@article{davies2014random,
	author = {Davies, Alex and Ghahramani, Zoubin},
	journal = {arXiv preprint arXiv:1402.4293},
	title = {The random forest kernel and other kernels for big data from random partitions},
	year = {2014}}

@article{geurts2006extremely,
	author = {Geurts, Pierre and Ernst, Damien and Wehenkel, Louis},
	journal = {Machine learning},
	number = {1},
	pages = {3--42},
	publisher = {Springer},
	title = {Extremely randomized trees},
	volume = {63},
	year = {2006}}

@techreport{breiman2000ensemble,
	author = {Breiman, Leo},
	institution = {Department of Statistics},
	title = {Some infinity theory for predictor ensembles},
	type = {Technical Report},
	year = {2000}}

@article{scornet2016random,
	author = {Scornet, Erwan},
	journal = {IEEE Transactions on Information Theory},
	number = {3},
	pages = {1485--1500},
	publisher = {IEEE},
	title = {Random forests and kernel methods},
	volume = {62},
	year = {2016}}

@article{biau2016random,
	author = {Biau, G{\'e}rard and Scornet, Erwan},
	journal = {{TEST}},
	number = {2},
	pages = {197--227},
	publisher = {Springer},
	title = {A random forest guided tour},
	volume = {25},
	year = {2016}}

@article{biau2008consistency,
	author = {Biau, G{\'e}rard and Devroye, Luc and Lugosi, G{\"a}bor},
	journal = {Journal of Machine Learning Research},
	number = {9},
	title = {Consistency of random forests and other averaging classifiers.},
	volume = {9},
	year = {2008}}

@misc{pinelis2019,
	author = {Pinelis, Iosif},
	howpublished = {MathOverflow},
	note = {URL:https://mathoverflow.net/q/323697 (version: 2019-02-21)},
	title = {Probability of two points being divided by an high-dimensional hyperplane},
	year = {2019}}

@article{belkin2018overfitting,
	author = {Belkin, Mikhail and Hsu, Daniel J and Mitra, Partha},
	journal = {Advances in Neural Information Processing Systems},
	pages = {2300--2311},
	title = {Overfitting or perfect fitting? Risk bounds for classification and regression rules that interpolate},
	volume = {31},
	year = {2018}}

@book{breiman1984classification,
	author = {Breiman, Leo and Friedman, Jerome and Stone, Charles J and Olshen, Richard A},
	publisher = {CRC press},
	title = {Classification and regression trees},
	year = {1984}}

@article{belkin2019reconciling,
	author = {Belkin, Mikhail and Hsu, Daniel and Ma, Siyuan and Mandal, Soumik},
	journal = {Proceedings of the National Academy of Sciences},
	number = {32},
	pages = {15849--15854},
	publisher = {National Acad Sciences},
	title = {Reconciling modern machine-learning practice and the classical bias--variance trade-off},
	volume = {116},
	year = {2019}}

@article{bartlett2020benign,
	author = {Bartlett, Peter L and Long, Philip M and Lugosi, G{\'a}bor and Tsigler, Alexander},
	journal = {Proceedings of the National Academy of Sciences},
	number = {48},
	pages = {30063--30070},
	publisher = {National Acad Sciences},
	title = {Benign overfitting in linear regression},
	volume = {117},
	year = {2020}}

@inproceedings{belkin2019does,
	author = {Belkin, Mikhail and Rakhlin, Alexander and Tsybakov, Alexandre B},
	booktitle = {The 22nd International Conference on Artificial Intelligence and Statistics},
	organization = {PMLR},
	pages = {1611--1619},
	title = {Does data interpolation contradict statistical optimality?},
	year = {2019}}

@article{wyner2017explaining,
	author = {Wyner, Abraham J and Olson, Matthew and Bleich, Justin and Mease, David},
	journal = {The Journal of Machine Learning Research},
	number = {1},
	pages = {1558--1590},
	publisher = {JMLR. org},
	title = {Explaining the success of {A}da{B}oost and random forests as interpolating classifiers},
	volume = {18},
	year = {2017}}

@article{bartlett2007adaboost,
	author = {Bartlett, Peter L and Traskin, Mikhail},
	journal = {Journal of Machine Learning Research},
	number = {Oct},
	pages = {2347--2368},
	title = {AdaBoost is consistent},
	volume = {8},
	year = {2007}}

\iftoggle{arxiv}{% no checklist for arxiv
}{%
\section*{Checklist}% checklist for neurips
\begin{enumerate}
\item For all authors...
\begin{enumerate}
  \item Do the main claims made in the abstract and introduction accurately reflect the paper's contributions and scope?
    \answerYes{}
  \item Did you describe the limitations of your work?
    \answerYes{} See the Section~\ref{section:Discussion}.
  \item Did you discuss any potential negative societal impacts of your work?
    \answerNA{}
  \item Have you read the ethics review guidelines and ensured that your paper conforms to them?
    \answerYes{}
\end{enumerate}

\item If you are including theoretical results...
\begin{enumerate}
  \item Did you state the full set of assumptions of all theoretical results? 
    \answerYes{} This is the opening sentence of our first main technical section, Section~\ref{section:RH-kernel}.
        \item Did you include complete proofs of all theoretical results?
    \answerYes{}
\end{enumerate}

\item If you ran experiments...
\begin{enumerate}
  \item Did you include the code, data, and instructions needed to reproduce the main experimental results (either in the supplemental material or as a URL)?
    \answerNA{}
  \item Did you specify all the training details (e.g., data splits, hyperparameters, how they were chosen)?
    \answerNA{}
        \item Did you report error bars (e.g., with respect to the random seed after running experiments multiple times)?
    \answerNA{}
        \item Did you include the total amount of compute and the type of resources used (e.g., type of GPUs, internal cluster, or cloud provider)?
    \answerNA{}
\end{enumerate}

\item If you are using existing assets (e.g., code, data, models) or curating/releasing new assets...
\begin{enumerate}
  \item If your work uses existing assets, did you cite the creators?
    \answerNA{}
  \item Did you mention the license of the assets?
    \answerNA{}
  \item Did you include any new assets either in the supplemental material or as a URL?
    \answerNA{}
  \item Did you discuss whether and how consent was obtained from people whose data you're using/curating?
    \answerNA{}
  \item Did you discuss whether the data you are using/curating contains personally identifiable information or offensive content?
    \answerNA{}
\end{enumerate}

\item If you used crowdsourcing or conducted research with human subjects...
\begin{enumerate}
  \item Did you include the full text of instructions given to participants and screenshots, if applicable?
    \answerNA{}
  \item Did you describe any potential participant risks, with links to Institutional Review Board (IRB) approvals, if applicable?
    \answerNA{}
  \item Did you include the estimated hourly wage paid to participants and the total amount spent on participant compensation?
    \answerNA{}
\end{enumerate}

\end{enumerate}

}

\clearpage

% reset the page numbering and prefix "A" to the page numbering
\pagenumbering{arabic}% resets `page` counter to 1
\renewcommand*{\thepage}{A\arabic{page}}
% See here: https://tex.stackexchange.com/questions/59572/custom-page-numbering-for-appendix

\appendix

% Prefix A to the section numbering. However, do not show the section header.
\newcommand\invisiblesection[1]{%
  \refstepcounter{section}%
  \addcontentsline{toc}{section}{\protect\numberline{\thesection}#1}%
  \sectionmark{#1}}

\invisiblesection{Appendix}

% \section{\supplemental: }

% \hrule height 4 pt
%   \vskip 0.25in
%   \vskip -\parskip%
 \vbox{%
    \hsize\textwidth
    \linewidth\hsize
    \vskip 0.1in
    \centering
    {\Large\bf \supplemental:  
    Consistent Interpolating Ensembles\\ 
    via the Manifold-Hilbert Kernel\par}
    }
%       \vskip 0.29in
%   \vskip -\parskip
%   \hrule height 1pt
%   \vskip 0.09in%

% {\centering \LARGE\bf \supplemental:  Consistent Interpolating Ensembles via the Manifold-Hilbert Kernel}

\subsection{Basics of Riemannian Manifolds}
\label{section:basics-of-riemannian-manifolds}
In this section, we review the main concepts from Riemannian manifold theory essential to this work.
Our main references are \textcite{sakai1996riemannian} and \textcite{do1992riemannian}.
Throughout, $d \in \mathbb{N}$ denotes the dimension.
We use the word \emph{smooth} to mean infinitely differentiable.

\textbf{Manifolds.} A smooth \emph{manifold} $M$ of dimension $d$ is a Hausdorff, second countable topological space together with an \emph{atlas}:
a set $\mathtt{Atlas}:=\{(U_\alpha, \varphi_\alpha)\}_{\alpha \in A}$ where 
1).\ $\{U_\alpha\}_{\alpha \in A}$ is an open cover of $M$,
% $\bigcup_{i=1}^\infty U_i = M$,
2).\ 
 % the set $U_i \subseteq M$ is open and
for each $\alpha \in A$,
$\varphi_\alpha : U_\alpha \to \varphi_\alpha(U_\alpha) \subseteq \mathbb{R}^d$ is a homeomorphism onto its image, and 
3).\  
 $\varphi_\alpha \circ \varphi_\beta^{-1} : \varphi_\beta(U_\alpha \cap U_\beta) \to \varphi_\alpha^{-1}(U_\alpha \cap U_\beta)$ is smooth 
for each pair $\alpha,\beta \in A$.
 An element $(U,\varphi)$ of $\mathtt{Atlas}$ is called a \emph{chart}.
 
  \textbf{Smooth maps.}
  A real-valued function $f : M \to \mathbb{R}$ is a \emph{smooth function} if $f \circ \varphi^{-1}$ is smooth (in the elementary calculus sense) for all charts $(U,\varphi)$.
  The set of all smooth functions is denoted $\fn(M)$, which forms an $\mathbb{R}$-vectorspace.
  Let $N$ be another smooth manifold with atlas $\mathcal{B}$.
  A function $\Phi : M \to N$ is a \emph{smooth map} if $g \circ \Phi \in \fn(M)$ for all $g \in \fn(N)$.

\textbf{Tangent space.}
Let $x \in M$.
A \emph{derivation at $x$} is a linear function $v : \fn(M) \to \mathbb{R}$ satisfying
the \emph{product rule}: $v[fg] = f(x) v[f] + g(x) v[g]$  for all $f,g \in \fn(M)$.
The \emph{tangent space at $x$}, denoted $T_xM$, is the vector space of all derivations at $x$.
Elements of $T_xM$ are referred to as \emph{tangent vectors} at $x$.
For a given chart $(U,\varphi)$ where $x \in U$, define a derivation at $x$, denoted $\pd{i}{x}$, by $f \mapsto \frac{d (f \circ \varphi^{-1})}{dz_i}( \varphi(x))$
where $\frac{d}{dz_i}$ is the $i$-th partial derivative in ordinary calculus.
It is a fact that $\{\pd{i}{x}: i = 1,\dots, d\}$ is a basis for $T_xM$.

Although the above definition of a tangent vector is abstract, it can be concretely interpreted in terms of derivative along a curve. Let $a<t_0<b$ be real numbers.
A \emph{curve through $x$} is a smooth map $\gamma : (a,b) \to M$ such that $\gamma(t_0) = x$.
Then 
$\fn(M) \ni f \mapsto \frac{d}{dt} f(\gamma(t))|_{t=t_0} \in \mathbb{R}$
defines a derivation at $x$.
Oftentimes, this derivation is denoted $\dot{\gamma}(t_0) \in T_xM$

% \textbf{Vector fields.}
\textbf{Riemannian metric.} 
The \emph{tangent bundle} is the set $TM := \bigcup_x T_xM$, which itself is a smooth manifold of dimension $2d$.
% Let $\tau_M : T M \to M$ be the projection function such that $\tau_M(T_xM) = x$.
% Then $\tau_M$ is a smooth map between manifolds.
A \emph{vector field on $M$} is a smooth map $\mathsf{V} : M \to TM$ such that $\mathsf{V}(x) \in T_xM$ for all $x \in M$.
The set of all vectors fields on $M$ is denoted $\vf(M)$.

A \emph{Riemannian metric} on $M$ is a choice of an inner product $\langle \cdot, \cdot \rangle_x$ 
(and thus, a norm $\|\cdot\|_x$)
on $T_xM$
for each $x \in M$ such that the function $M \to \mathbb{R}$ given by $x \mapsto \langle\mathsf{V}(x), \mathsf{U}(x)\rangle_x$ is smooth for all $\mathsf{V},\mathsf{U} \in \vf(M)$.
As shorthands, when $x$ is clear from context, we drop the subscripts and simply write $\langle\cdot,\cdot\rangle$ and $\|\cdot \|$ instead.
Choosing an orthonormal basis for $T_xM$ with respect to $\langle \cdot, \cdot \rangle_x$ for each $x$, we can identify $T_xM$ with $\mathbb{R}^d$ with the ordinary dot inner product.

% \textbf{Coordinate representations.}
Let $x \in M$ and $(U,\varphi)$ be a chart such that $x \in U$.
Define $g_{ij}(x) = \langle \pd{i}{x}, \pd{j}{x}\rangle_x$.
Denote by $G(x)$ the $d\times d$ positive definite matrix $[g_{ij}(x)]_{ij}$.
Below, we will refer to the function $G: U \to \mathbb{R}^{d\times d}$ as the \emph{coordinate representation} of the Riemannian metric.
Define $g^{ij}(x) := [G(x)^{-1}]_{ij}$.
The \emph{Christoffel symbols} with respect to $(U,\varphi)$ are defined by
$\Gamma^k_{ij} := \frac{1}{2}\sum_{\ell=1}^d g^{k\ell}(\pd{i}{x} g_{j\ell} + \pd{j}{x} g_{i\ell} - \pd{\ell}{x} g_{ij})$.
Note that $g_{k\ell}$,  $g^{k\ell}$, $G$, $\Gamma^k_{ij}$, and $\pd{i}{x} g_{j\ell}$ are all functions with domain $U$.

\textbf{Geodesics.}
Fix a chart $(U,\varphi)$.
Consider a smooth curve $\gamma: [a,b] \to U$. Let $\zeta_i(t) := [\varphi(\gamma(t))]_i$ be the $i$-th component functions.
The curve $\gamma$ is a \emph{geodesic} if $\zeta$ is a solution to the following system of second order ordinary differential equations (ODEs):
$\frac{d^2\zeta_i}{dt^2}  + \sum_{j,\ell=1}^d\Gamma^i_{j\ell} \circ \gamma \frac{d\zeta_j}{dt} \frac{d\zeta_\ell}{dt} = 0$ for all $i = 1,\dots, d$ at all time $t \in [a,b]$.

Geodesics are minimizers of the so-called 
\emph{energy functional}
$E(\gamma) = \frac{1}{2}\int_{a}^{b} \|\dot{\gamma}(t)\|_{\gamma(t)}^2 dt$.
The above system of ODEs are the analog of the ``first derivative test'' for local minimizers of $E$.
Thus, geodesics are defined independently of the choice of the chart.

\textbf{Exponential map.}
For $x \in M$ and $v \in T_xM$, there exists $\epsilon >0$ and a unique geodesic curve $\gamma_v : [-\epsilon, \epsilon] \to M$ such that $\gamma_v(0) = x$ and $\dot{\gamma}_v(0) = v$.
This follows from the existence and uniqueness of the solution to an ODE given initial conditions where the ODE is as discussed above.
Note that although geodesics are previously defined in $U$ where $(U,\varphi)$ is a chart, they can be extended outside of $U$ using additional charts.

% Let $U_xM = \{ v \in T_xM : \|v\|_x = 1\}$ be the set of unit tangent vectors.
% Let $v \in T_xM$ be fixed and write $v = t u$ where $u =u(v) := v/\|v\|_x$ and $t = t(v) := \|v\|_x$.

% Now, fix a $u \in U_xM$. Let $\epsilon >0$ be such that $\gamma_{u} : [-\epsilon,\epsilon] \to M$ is a geodesic as in the preceding paragraph.
Let $x \in M$ and $v \in T_xM$ be fixed and let $\gamma_v : [-\epsilon,\epsilon] \to M$ be as in the preceding paragraph.
If $\|v\|_x \le \epsilon$, then define
$\exp_x(v) := \gamma_{v}(1)$.
A fundamental fact is that $\exp_x$, known as the \emph{exponential map at $x$}, can be defined on an open set of $T_xM$ containing the origin.

\textbf{Distance function.}
Let $x, \xi \in M$ and $a<b$ be real numbers.
A \emph{piecewise smooth curve} from $x$ to $\xi$  is a
piecewise smooth map $\gamma : [a,b] \to M$ such that $\gamma(a) = x$ and $\gamma(b) = \xi$.
Assume that $M$ is connected.
Then for all $x, \xi \in M$, there exists a piecewise smooth curve from $x$ to $\xi$.
The \emph{length of $\gamma$} is defined as $\len(\gamma) := \int_a^b \|\dot{\gamma}(t)\|_{\gamma(t)} dt$.
Define $\dist_M(x,\xi) := \inf \{ \len(\gamma) : \gamma$ is a piecewise smooth curve from $x$ to $\xi\}$, which is a metric on $M$ in the sense of metric spaces
(see \cite[Proposition 1.1]{sakai1996riemannian}).
For $x \in M$ and $r \in (0,\infty)$, the open ball centered at $x$ of radius $r$ is denoted  
$\Ball{M}{x}{r} := \{z \in M : \dist_{M}(x,z) < r\}$.

% \textbf{Normal coordinate system.}
% Let $D_x \subseteq T_xM$ be an open set containing $0$ such that $\exp_x(D_x)$ is open in $M$ and $\exp_x : D_x \to \exp_x(D_x) \subseteq M$ is a diffeomorphism onto its image.

\textbf{Complete Riemannian manifolds.}
A Riemannian manifold is \emph{complete} if it is a complete metric space under the metric $\dist_M$.
The 
Hopf-Rinow theorem (\cite[Ch.\ 8, Theorem 2.8]{do1992riemannian})
states that if $M$ is connected and complete, then the exponential $\exp_x$ can be defined on the entire $T_xM$.
% When $M$ is a manifold, we let $d \in \mathbb{N}$ denote the dimension of $M$.

\subsection{Proof of Lemma~\ref{lemma:Riemannian-logarithm}}
\label{appendix:proof-of-lemma:Riemannian-logarithm}
This section uses definitions and notations introduced in 
Section~\ref{section:riemannian-logarithm}. In particular, recall the cut locus $C_x$, the tangent cut locus $\tilde{C}_x$, the interior set $I_x$ and the tangent interior set $\tilde{I}_x$.
The proof of Lemma~\ref{lemma:Riemannian-logarithm} is presented towards the end of the section.
At this point, we compile some facts from various sources about the cut locus.
% Results cited in Lemma~\ref{lemma:facts} from \cite{sakai1996riemannian} are all from Chapter II therein.
\begin{lemma}
  \label{lemma:facts}
For all $x \in M$, we have
\begin{compactenum}
  \item 
    $C_x$ is a closed subset of $M$ (\textcite[Proposition 1.2]{hebda1987parallel}).
  \item $I_x \cap C_x = \emptyset$ and $I_x \cup C_x = M$
  (\textcite[Ch II, Lemma 4.4 (1)]{sakai1996riemannian})
\item $I_x$ is an open subset of $M$ (immediate from 1 and 2 above)
\item 
\label{item:expx-is-diffeo}
  $\exp_x : \tilde{I}_x \to I_x$ is a diffeomorphism 
  (\cite[Ch II, Lemma 4.4 (2)]{sakai1996riemannian})
\item 
\label{item:Cx-measure-zero}
  $\lambda_M(C_x) = 0$, where $\lambda_M$ is the Riemann-Lebesgue measure (\cite[Lemma 4.4 (3)]{sakai1996riemannian})
\item $\tau_x$ is continuous and $\inf_{u \in U_xM} \tau_x(u) > 0$ (\cite[Ch II, Propositions 4.1 (2) and 4.13 (1)]{sakai1996riemannian})
\end{compactenum}
\end{lemma}

While the following lemma is elementary, we provide a proof since we could not find one in the literature.

\begin{lemma}
\label{lemma:tangent-disc}
  For all $x \in M$, the (topological) closure of $\tilde{I}_x$ in $T_xM$ is 
  $\tilde{D}_x$.
  Furthermore, for all $x \in M$, we have $\exp_x(\tilde{D}_x) = M$.
\end{lemma}
\begin{proof}[Proof of Lemma~\ref{lemma:tangent-disc}]
Take a convergent sequence $\{t_i u_i\}_{i \in \mathbb{N}} \subseteq \tilde{I}_x$ where $u_i \in U_xM$ and $0 \le t_i < \tau_x(u_i)$.
  Let $v^* = \lim_i t_i u_i$.
  Our goal is to show that $v^* \in \tilde{D}_x=\tilde{I}_x \cup \tilde{C}_x$.

  Since $U_xM$ is compact, we may assume that $u^*:=\lim_i u_i$ exists after passing to a subsequence if necessary.
  Furthermore, $\|t_i u_i \|_x = t_i$ implies that $t^*:=\lim_i t_i$ exists as well (i.e., $t^* < \infty$).
  Hence, $v^* = t^* u^*$.

  Consider the case that $\tau_x(u^*) = \infty$. Then $0\le t^* <\tau_x(u^*)$ implies that $v^* = t^* u^* \in \tilde{I}_x$.
  For the other case that $t(u) < \infty$, we first note that $t_i u_i \in \tilde{I}_x$ implies that $t_i < \tau_x(u_i)$. Taking the limit of both sides, we have
  $
    t^* = \lim_i t_i \le \lim_i \tau_x(u_i) = \tau_x(u^*).
  $
  Note that the last limit can be exchanged since $\tau_x$ is continuous (Lemma~\ref{lemma:facts} part 6).
  Thus, either $t^* < \tau_x(u^*)$ in which case $v^* \in \tilde{I}_x$, or $t^* = \tau_x(u^*)$ in which case $v^* = \tau_x(u^*) u^* \in \tilde{C}_x$.
  
  For the ``furthermore'' part, note that
  \[ % &
    \exp_x(\tilde{D}_x) 
  % \\&=
    =
  \exp_x(\tilde{I}_x \cup \tilde{C}_x) 
  % \\&=
  =
  \exp_x(\tilde{I}_x) \cup \exp_x(\tilde{C}_x) 
  % \\&=
  =
  {I}_x \cup {C}_x 
  % \\&=
  =
  M 
  \]
  where the last equality is Lemma~\ref{lemma:facts} part 2.
\end{proof}

\begin{proof}[Proof of Lemma~\ref{lemma:Riemannian-logarithm}]
Denote by $\mathrm{cl}(T_xM)$ the set of closed subsets of $T_xM$.
  Define $\psi : M \to \mathrm{cl}(T_xM)$ by 
  $\psi(\xi) := \{ x \in \tilde{D}_x: \exp_x(x) = \xi\} = \exp_x^{-1}(\xi)\cap \tilde{D}_x.$ Note that $\psi(\xi)$ is a closed set by Lemma~\ref{lemma:tangent-disc}.
  
  We claim that $\psi$ is \emph{weakly-measurable}, i.e., for every open set $\tilde{U} \subseteq T_xM$, the subset of $M$ defined by
  $
    \{ \xi \in M : \psi(\xi) \cap \tilde{U} \ne \emptyset\}
  $
  is Borel.
  To see this, note that 
  \begin{align*}
    &\{ \xi \in M : \psi(\xi) \cap \tilde{U} \ne \emptyset\}
    \\&=
    \{ \xi \in M : \exp_x^{-1}(\xi) \cap \tilde{D}_x\cap \tilde{U} \ne \emptyset\}
    \\&=
    \{ \xi \in M : \exp_x(\tilde{D}_x \cap \tilde{U}) \ni \xi\}
    \\&=
    \exp_x(\tilde{D}_x \cap \tilde{U}).
  \end{align*}
  As inner product spaces, $T_xM$ and $\mathbb{R}^d$ are isomorphic (see 
  Section~\ref{section:basics-of-riemannian-manifolds}-{Riemannian metric}).
  Since, $T_xM$ and $\mathbb{R}^d$ are homeomorphic as topological spaces, $\mathbb{R}^d$ being locally compact implies $T_xM$ is locally compact as well.
  Thus, we can write $\tilde{U} = \bigcup_{i \in \mathbb{N}} \tilde{K}_i$ as a countable union of compact sets $\tilde{K}_i \subseteq T_xM$.
  Furthermore, $\tilde{D}_x \cap \tilde{U} = \bigcup_{i \in \mathbb{N}} \tilde{D}_x \cap \tilde{K}_i$
  and so $\exp_x(\tilde{D}_x \cap \tilde{U}) = \bigcup_{i \in \mathbb{N}} \exp_x(\tilde{D}_x \cap \tilde{K}_i)$.

  Since $\exp_x$ is continuous, $\exp_x(\tilde{D}_x \cap \tilde{K}_i)$ is a compact subset of $M$, and hence closed and bounded by the Hopf-Rinow theorem (\cite[Ch.\ 8, Theorem 2.8]{do1992riemannian}).
  Thus, $\exp_x(\tilde{D}_x \cap \tilde{U}) = \bigcup_{i \in \mathbb{N}} \exp_x(\tilde{D}_x \cap \tilde{K}_i)$ is a countable union of closed sets, which is Borel.
  This proves the claim that $\psi$ is weakly Borel measurable.

  By the Kuratowski–Ryll-Nardzewski measurable selection theorem (see \cite[Theorem 6.9.3]{bogachev2007measure}), there exists a Borel measurable function $M \to T_xM$, which we denote by $\log_x$, such that $\log_x(\xi) \in \psi(\xi) = \exp_x^{-1}(\xi)$ for all $\xi \in M$, as desired.
  % To conclude, recall that a Borel measurable function is, \emph{a fortiori}, Lebesgue measurable.
  % In other words, $\log_x$ is a right inverse of $\exp_x$, i.e., $\exp_x \circ \log_x$ is the identity.
  By construction, $\log_x(\xi) \in \exp_x^{-1}(\xi)$ for all $\xi \in M$, and so $\exp_x(\log_x(\xi)) = \xi$ is immediate.

  For the ``furthermore'' part, let $\xi \in M$ be arbitrary and let $z := \log_x(\xi) \in \tilde{D}_x$.
  Let $\{z_i\} \subseteq \tilde{I}_x$ be a sequence such that $\lim_{i} z_i = z$.
  By Equation~\eqref{equation:exponential-map-preserves-distance}, we have $\dist_M(x,\exp_x(z_i)) = \|z_i\|_x$.
  By continuity of $\dist_M$ and $\exp_x$, we have $
  \dist_M(x,\xi) = \dist_M(x,\exp_x(z))= \lim_i \dist_M(x,\exp_x(z_i))$.
  To conclude, we have
  $
  \lim_i
  \dist_M(x,\exp_x(z_i)) = 
    \lim_i \|z_i\|_x = \|z\|_x
    =
    \|\log_x(\xi)\|_x$, as desired.
\end{proof}

\subsection{Proof of Proposition~\ref{proposition:push-foward-distribution}}
\label{appendix:proof-of-proposition:push-foward-distribution}

Recall from Section~\ref{section:basics-of-riemannian-manifolds}-Riemannian metric, given a chart $(U,\varphi)$, one can define the matrix-valued function $G : U \to \mathbb{R}^{d\times d}$ referred to earlier as the coordinate representation of the Riemannian metric. 
Now, Lemma~\ref{lemma:facts} part 3 states that $I_x$ is an open neighborhood of $x$. Furthermore, $\tilde{I}_x$ is an open subset of $T_xM$, which is identified with $\mathbb{R}^d$ using an orthonormal basis 
(see Section~\ref{section:basics-of-riemannian-manifolds}-Riemannian metric).
Hence, $\{(I_x, \log_x|_{I_x})\}_{x \in M}$ is an atlas of $M$
(see Section~\ref{section:basics-of-riemannian-manifolds}-Manifolds).

\begin{definition}
  \label{definition:normal-coordinate-system}
The chart $(I_x,\log_x|_{I_x})$ is called a \emph{normal coordinate system} at $x$. Let $G : I_x \to \mathbb{R}^{d \times d}$ be the coordinate representation of the Riemannian metric for this chart. To emphasize the dependency on $x$, we write $G_x := G$.
Denote by $G_x^{\perp} : M \to \mathbb{R}^{d\times d}$ the zero extension of $G_x$ to the rest of $M$, i.e., $G_x^\perp(\xi) = G_x(\xi)$ for $\xi \in I_x$ and $G_x^\perp(\xi)$ is the zero matrix for $\xi \not \in I_x$.
\end{definition}

The normal coordinate system has the property that 
$G_x(x) = G_x^\perp(x)$ is the identity matrix. This is the result of \textcite[Ch.\ II \S 2 Exercise 4]{sakai1996riemannian}.
% Let $\chi_{\tilde{I}_x}(z) := \mathbb{I}\{z \in \tilde{I}_x\}$ be the characteristic function for the set $\mathcal{I}_x$.

\begin{lemma}[Change-of-Variables]
  \label{lemma:change-of-variables}
  Let $x \in M$ be fixed. Define the function $\nu_x : M \to \mathbb{R}$ by 
  $\nu_x(\xi) = \sqrt{|\det G_x^\perp(\xi)|}$ where $G_x^\perp$ is as in Definition~\ref{definition:normal-coordinate-system}. Then $\nu_x$ is Borel-measurable.
Furthermore, $\nu_x$ satisfies the following property: Let  $f: M \to \mathbb{R}$ be an absolutely integrable function. Define the function
  \[h : T_xM \to \mathbb{R}
  \quad \mbox{by}\quad
  h(z):= 
  f (\exp_x(z))  \cdot \nu_x(\exp_x(z)).
\] Then 
\emph{(i)} $h(0) = f(x)$ and \emph{(ii)}
for all Borel set $\tilde{B} \subseteq T_xM$ we have $\int_{B} f d \lambda_M =
\int_{\tilde{B}} 
% f \circ \exp_p  \sqrt{|G| \circ \exp_p} 
  h
d \lambda$
where $B:=\exp_x(\tilde{B} \cap \tilde{I}_x)$.
% for all Borel sets $\tilde{B} \subseteq T_pM$.
\end{lemma}

\begin{proof}[Proof of Lemma~\ref{lemma:change-of-variables}]
We first show that $\nu_x$ is Borel-measurable. Recall that $G_x^\perp : M \to \mathbb{R}^{d\times d}$ is the zero extension of $G_x : I_x \to \mathbb{R}$, which is by definition smooth (see Section~\ref{section:basics-of-riemannian-manifolds}-Riemannian metric). In particular, $G_x : I_x \to \mathbb{R}$ is continuous and so $\sqrt{\det(G_x(\bullet))}$ is Borel-measurable. Now, note that $\sqrt{\det(G_x^\perp(\bullet))}$ is the zero extension of $\sqrt{\det(G_x(\bullet))}$ from $I_x$ to $M$. Hence, $\sqrt{\det(G_x^\perp(\bullet))}$, which is $\nu_x$ by definition, is Borel-measurable.

Next, we prove the ``Furthermore'' part (i). 
Note that $\exp_x(0) = x$. 
Moreover, $G_x^\perp(x)=G_x(x)$ is the identity matrix as asserted after Definition~\ref{definition:normal-coordinate-system} (see  \textcite[Ch.\ II \S 2 Exercise 4]{sakai1996riemannian}).
 Thus,
  $h(0) =
  f(\exp_x(0)) \sqrt{|\det G_x^\perp(\exp_x(0))|}
  =
  f(x) \sqrt{1}
  =
  f(x)
  $, as desired.

  For the ``Furthermore'' part (ii), we first
  note that  $\tilde{B} = (\tilde{B} \cap \tilde{I}_x) \cup (\tilde{B} \cap \tilde{C}_x)$ expresses $\tilde{B}$ as a disjoint union.
  Thus,
  $B=\exp_x(\tilde{B}) 
  = \exp_x(\tilde{B} \cap \tilde{I}_x) \cup \exp(\tilde{B} \cap \tilde{C}_x)
  $ 
  expresses $B$ as a disjoint union as well.
  Moreover, $\exp(\tilde{B} \cap \tilde{C}_x) \subseteq \exp(\tilde{C}_x) = C_x$, which has $\lambda_M$-measure zero (Lemma~\ref{lemma:facts} part 5).
  % Let $\tilde{B} := \log_x(B)$.
  % \begin{align*}
  % \end{align*}

Recall that $\lambda$ is the shorthand for the ordinary Lebesgue measure $\lambda_{\mathbb{R}^d}$ (see paragraph right after Definition~\ref{definition:manifold-Hilbert-kernel}).
  Now, we directly compute to obtain the formula
  \begin{align*}
    \int_{\tilde{B}} h d\lambda 
  &=
  \int_{\tilde{B} \cap \tilde{I}_x} f \circ \exp_x \sqrt{|\det (G_x^\perp \circ \exp_x)|} d \lambda
  \\&=
  \int_{\log_x(\exp_x(\tilde{B} \cap \tilde{I}_x))} f \circ \exp_x \sqrt{|\det (G_x^\perp \circ \exp_x)|} d \lambda
  \\&=
  \int_{\exp_x(\tilde{B} \cap \tilde{I}_x)} f d\lambda_M
  \qquad \because \mbox{\textcite[Ch XII, Thm 1.10]{amann2009integration}}
  \\&=
  \int_{\exp_x(\tilde{B} \cap \tilde{I}_x)} f d\lambda_M
  +
  \int_{\exp_x(\tilde{B} \cap \tilde{C}_x)} f d\lambda_M
  \\&=
  \int_B f d \lambda_M,
  \end{align*}
  as desired.
\end{proof}

\begin{proposition}
  \label{proposition:logX-random-var}
  Let $x \in M$ be fixed. 
  Let $X$ be a random variable on $M$ with density $f_X$
  where the underlying probability space is 
  $(\Omega, \mathbb{P}, \mathcal{A})$ (see Definition~\ref{definition:probability-space}).
  Define $Z := \log_x(X)$.
  Then
  $Z$
  is a random variable on $T_xM$ such that 
  for all events $E \in \mathcal{A}$ and Borel sets $\tilde{B} \subseteq T_xM$ we have
  $\Pr(E \cap \{ Z \in \tilde{B}\}) = \Pr(E \cap \{ X \in \exp_x(\tilde{B} \cap \tilde{I}_x)\})$,  
\end{proposition}
\begin{proof}[Proof of Proposition~\ref{proposition:logX-random-var}]
To start with, we have 
  \begin{align*}
    & 
    \Pr(E \cap \{Z \in \tilde{B}\})
    \\ &=
    \Pr(E \cap \{Z \in \tilde{B} \cap \tilde{D}_x\}) \qquad \because \log_x(M) \subseteq \tilde{D}_x
    \\ &=
    \Pr(E \cap \{Z \in \tilde{B} \cap \tilde{I}_x\}) +  \Pr(E \cap \{Z \in \tilde{B} \cap \tilde{C}_x\}) 
    \qquad \because 
    \tilde{D}_x = \tilde{I}_x \cup \tilde{C}_x,\,
    \emptyset = \tilde{I}_x \cap \tilde{C}_x
    \\ &=
    \Pr(E \cap \{\log_x(X) \in \tilde{B} \cap \tilde{I}_x\}) +  \Pr(E \cap \{\log_x(X) \in \tilde{B} \cap \tilde{C}_x\}).
  \end{align*}
  Since $\exp_x : \tilde{I}_x \to I_x$ is a diffeomorphism (Lemma~\ref{lemma:facts}-part~\ref{item:expx-is-diffeo}) with inverse $\log_x$, we have 
  \[
E \cap \{\log_x(X) \in \tilde{B} \cap \tilde{I}_x\}
=
E \cap \{X \in \exp_x(\tilde{B} \cap \tilde{I}_x)\}
  \]
  as sets.
  On the other hand, 
  \[ E \cap \{\log_x(X) \in \tilde{B} \cap \tilde{C}_x\}
    \subseteq \{ X \in {C}_x\}.
  \] 
  Finally, $\Pr(X \in C_x) = \int_{C_x} f_X d\lambda_M = 0$ since $C_x$ has $\lambda_M$-measure zero
  (Lemma~\ref{lemma:facts}-part~\ref{item:Cx-measure-zero}).
\end{proof}

\begin{proof}[Proof of Proposition~\ref{proposition:push-foward-distribution} part (\ref{item:logX})]
Recall that $\lambda$ is the shorthand for the ordinary Lebesgue measure $\lambda_{\mathbb{R}^d}$ (see paragraph right after Definition~\ref{definition:manifold-Hilbert-kernel}).
Let $E = \Omega$ in Proposition~\ref{proposition:logX-random-var}. Then we have 
  \begin{align*}
    &\Pr(Z \in \tilde{B})
    \\ &=
    \Pr(X \in \exp_x(\tilde{B} \cap \tilde{I}_x))
    \qquad \because \mbox{Part (i)}
    \\ &=
    \int_{\exp_x(\tilde{B} \cap \tilde{I}_x)} f_X d\lambda_M
    \qquad \because \mbox{$f_X$ is the density of $X$}
    % \\ &=
    % \int_{\log_x(\exp_x(\tilde{B} \cap \tilde{I}_x))} f_X\circ \exp_x \sqrt{|G| \circ \exp_x} d\lambda
    \\ &=
    \int_{\tilde{B} \cap \tilde{I}_x} 
% \chi_{\tilde{I}_x}\cdot
    (f_X\circ \exp_x) \cdot(\nu_x \circ \exp_x)d\lambda
    \qquad \because \mbox{Lemma~\ref{lemma:change-of-variables}}
    \\ &=
    \int_{\tilde{B} } 
    % \chi_{\tilde{I}_x} \cdot f_X\circ \exp_x \sqrt{|G| \circ \exp_x} 
    f_Z
    d\lambda
    \qquad \because \mbox{Definition of $f_Z$}
  \end{align*}
  By assumption, $f_X$ is Borel-measurable. 
  By Lemma~\ref{lemma:change-of-variables}, $\nu_x$ is Borel-measurable.
  Since $\exp_x$ is continuous, we have that
  both
  $f_X\circ \exp_x$ and $\nu_x \circ \exp_x$ are Borel-measurable. This proves that $f_Z$ is Borel-measurable.
  Hence, the integrand is Borel-measurable and a density function for $Z$.
  % We first need to prove that $\varphi_*^{-1} \sigma \ll \lambda$.
  % We now prove that 
  % $
  %   f \circ \varphi \frac{d( \varphi_*^{-1}\sigma)}{d\lambda}
  %   $
  %   is in $L^1(\mathbb{R}^d,\lambda)$.
\end{proof}

\begin{proof}[Proof of Proposition~\ref{proposition:push-foward-distribution} part (\ref{item:mapping-lebesgue-points})]
Recall that $\lambda$ is the shorthand for the ordinary Lebesgue measure $\lambda_{\mathbb{R}^d}$ (see paragraph right after Definition~\ref{definition:manifold-Hilbert-kernel}).
By Lemma~\ref{lemma:facts} part 6, we have $\tau_x^*:=\inf_{u \in U_xM} \tau_x(u) >0$.
  Now, let $r \in (0,\tau_x^*)$.
  % The ball around point $p$ is given by
  % \[
  %   \Ball{M}{x}{r} := \{ \xi \in M: \dist(x, \xi) < r\}.
  % \]
  By the definition of $r$, we have $\Ball{M}{x}{r} \subseteq \tilde{I}_x$.
  Hence 
  letting $z = \log_x(\xi)$ for $\xi \in \Ball{M}{x}{r}$, 
  by
Equation~\eqref{equation:exponential-map-preserves-distance}
   we have 
  \begin{equation}
    \label{equation:angle-to-norm}
    \dist_M(x,\xi) = \dist_M(x, \exp_x(z)) = \|z\|_x.
  \end{equation}
  Thus, 
  \begin{equation}
    \label{equation:change-of-balls}
    \log_x(\Ball{M}{x}{r}) = \{ z \in T_xM: \|z \|_x < r\}
    =
  \Ball{T_xM}{0}{r}
\end{equation}
and
  \begin{equation}
    \Ball{M}{x}{r} 
    =
    \exp_x(\Ball{T_xM}{0}{r}).
\end{equation}
  % Now, by the change of variables formula, we have 
  % \begin{align*}
  %   % \int_{\Ball{\mathbb{S}^d}{x}{r}} f(\xi) \sigma(d \xi)
  %   &\int_{\Ball{M}{x}{r}} f  d \lambda_M
  % \\& =
  % \int_{\Ball{M}{x}{r}} f  d ((\exp_x)_* (\log_x)_* \lambda_M )
  % \\& =
  %   \int_{\varphi^{-1}(B(x,r))} 
  %   f \circ \varphi
  %   d(\varphi_*^{-1} \sigma)
  %   \quad \because \mbox{change-of-variables formula}
  % \\& =
  % \int_{\Ball{T_x(\mathbb{S}^d)}{0}{r}} 
  %   f \circ \varphi
  %   d(\varphi_*^{-1} \sigma)
  %   \quad \because \mbox{\eqref{equation:change-of-balls}}
  % \\& =
  % \int_{\Ball{T_x(\mathbb{S}^d)}{0}{r}} 
  %   f \circ \varphi
  %   \frac{d(\varphi_*^{-1} \sigma)}{d\lambda} 
  %   d\lambda
  %   \quad \because \mbox{Radon-Nikodym theorem}
  % \\& =
  % \int_{\Ball{T_x(\mathbb{S}^d)}{0}{r}} 
  % g
  %   d\lambda
  % \end{align*}
Thus, by Lemma~\ref{lemma:change-of-variables},  we have
  \begin{equation}
    \label{equation:change-of-variables-detailed}
    \int_{\Ball{M}{x}{r}} f  d \lambda_M
    =
  \int_{\Ball{T_xM}{0}{r}} 
  h
    d\lambda.
  \end{equation}
  Before proceeding, we need the following lemma:
\begin{lemma}
  \label{lemma:volume-asymptotic-equivalence}
  For all $x \in M$, we have 
$
\lim_{r \to 0} \frac{
\lambda_M(\Ball{M}{x}{r})
}
{
\lambda(\Ball{T_xM}{0}{r})
} 
= 1
$.
\end{lemma}
\begin{proof}[Proof of Lemma~\ref{lemma:volume-asymptotic-equivalence}]
  Let
  $\omega_d := \pi^{d/2}/\Gamma(\tfrac{d}{2} + 1)$ be the volume of the unit ball in $\mathbb{R}^d$ where $\Gamma$ is the gamma function. Then $
\lambda(\Ball{T_xM}{0}{r}) = \omega_d r^d$.
  Next, \cite[Ch II.5 Exercise 3]{sakai1996riemannian} states that
  \[
    \lim_{r \to 0} \frac{r^d \omega_d - \lambda_M(\Ball{M}{x}{r})}{r^{d+2}} = \frac{\omega_d}{6(d+2)} S_x
  \]
  where $S_x \in \mathbb{R}$ is a constant that depends only on $x$ (it is the scalar curvature of $M$ at $x$).
  By simple algebra, the above yields
  \[
    0=\lim_{r \to 0}  \frac{1}{r^2} \left(1 - \frac{\lambda_M(\Ball{M}{x}{r})}{\omega_d r^d} - \frac{S_x r^2}{6(d+2)}\right)
  \]
  In particular, we have $\lim_{r \to 0}1 - \frac{\lambda_M(\Ball{M}{x}{r})}{\omega_d r^d} = 0$, as desired.
\end{proof}
% \begin{proof}
% Using \eqref{equation:change-of-variables-detailed} with $f \equiv 1$, the constant function with value $1$, we have
% \[
%   \sigma(\Ball{\mathbb{S}^d}{x}{r})
%   =
%     \int_{\Ball{\mathbb{S}^d}{x}{r}}   d \sigma
%     =
%   \int_{\Ball{T_x(\mathbb{S}^d)}{0}{r}} 
%     \frac{d(\varphi_*^{-1} \sigma)}{d\lambda} 
%     d\lambda
% \]
% Letting $\underline{M}(r) := \inf_{z \in \Ball{T_x(\mathbb{S}^d)}{0}{r}} 
% |\frac{d(\varphi_*^{-1} \sigma)}{d\lambda}(z)|$ and 
% $\overline{M}(r) := \sup_{z \in \Ball{T_x(\mathbb{S}^d)}{0}{r}} 
% |\frac{d(\varphi_*^{-1} \sigma)}{d\lambda}(z)|$, we have
% \[
%   \underline{M}(r)
% \lambda(\Ball{T_x(\mathbb{S}^d)}{0}{r})
%   \le
%   \sigma(\Ball{\mathbb{S}^d}{x}{r})
%   \le
%   \overline{M}(r)
% \lambda(\Ball{T_x(\mathbb{S}^d)}{0}{r}).
% \]

% TODO: $\lim_{r \to 0} \underline{M}(r) = 1$ and $\lim_{r \to 0} \overline{M}(r) = 1$.
% This is obvious since the exponential map has Jacobian equal to the identity at $0$ and furthermore the exponential map is continuous differentiable. (infinitely differentiable, in fact).
% \end{proof}

  Now we continue with the proof of  Proof of Proposition~\ref{proposition:push-foward-distribution} part~(\ref{item:mapping-lebesgue-points}). We observe that
  \begin{align*}
    f(x) 
    &= \lim_{r \to 0} 
    \frac{
      \int_{\Ball{M}{x}{r}} f d \lambda_M
  }
    {
      \lambda_M(\Ball{M}{x}{r})
    % \int_{\Ball{M}{x}{r}}
    % d\lambda_M
  }
  \quad \because \mbox{$x$ is a Lebesgue point of $f$}
  \\&=
     \lim_{r \to 0} 
    \frac{
  \int_{\Ball{T_x(M)}{0}{r}} 
  h
    d\lambda
  }
    {
      \lambda_M(\Ball{M}{x}{r})
    % \int_{\Ball{M}{x}{r}}
    % d\lambda_M
  }
  \quad \because \mbox{definition of $h$ and equation \eqref{equation:change-of-variables-detailed}}
  \\&=
\lim_{r \to 0}
    \frac{
  \int_{\Ball{T_xM}{0}{r}} 
  h
    d\lambda
  }
    {
      \lambda_M(\Ball{M}{x}{r})
    % \int_{\Ball{M}{x}{r}}
    % d\lambda_M
  }
\frac{\lambda_M(\Ball{M}{x}{r})}{\lambda(\Ball{T_xM}{0}{r})} 
\quad \because\mbox{Lemma~\ref{lemma:volume-asymptotic-equivalence}}
\\&=
\lim_{r \to 0}
    \frac{
  \int_{\Ball{T_xM}{0}{r}} 
  h
    d\lambda
  }
{
\lambda(\Ball{T_xM}{0}{r})
}.
  \end{align*}
  Since $f(x) = h(0)$ (Lemma~\ref{lemma:change-of-variables})
  , we've shown that
  \[
    g(0) = 
\lim_{r \to 0}
    \frac{
  \int_{\Ball{T_xM}{0}{r}} 
  h
    d\lambda
  }
{
\lambda(\Ball{T_xM}{0}{r})
}.
  \]
  Thus, $0$ is a Lebesgue point of $h$, as desired.
\end{proof}

\subsection{Proof of Proposition~\ref{proposition:condition-expectation-transformation} }
\label{appendix:proof-of-proposition:condition-expectation-transformation}

% \begin{proof}[Proof of Proposition~\ref{proposition:condition-expectation-transformation} ]
  Recall that $\lambda$ is the shorthand for the ordinary Lebesgue measure $\lambda_{\mathbb{R}^d}$ (see paragraph right after Definition~\ref{definition:manifold-Hilbert-kernel}).
Let $A \subseteq \mathbb{R}$ and $\tilde{B} \subseteq T_xM$ be Borel subsets.
  Then 
  \begin{align*}
  &\int_{\tilde{B}} P_{Y| Z}(A|z) f_Z(z) d \lambda(z)
\\&=
\int_{\tilde{B}} P_{Y| X}(A|\exp_x(z)) f_Z(z) d \lambda(z)
\qquad \because{\mbox{Definition of $P_{Y|Z=z}$}}
\\&=
\int_{\exp_x(\tilde{B} \cap \tilde{I}_p)} P_{Y| X}(A|x) f_X(x) d \lambda_M(x)
\qquad \because{\mbox{Lemma~\ref{lemma:change-of-variables} and Proposition~\ref{proposition:push-foward-distribution} (ii)}}
\\&=
\Pr(Y \in A , X \in \exp_x(\tilde{B}\cap \tilde{I}_x))
\\&=
\Pr(Y \in A , Z \in \tilde{B})
\qquad \because{\mbox{Proposition~\ref{proposition:push-foward-distribution} (i) with $E := \{Y \in A\}$}}
\end{align*}
This proves that $P_{Y|Z}(\cdot|\cdot)$ is a conditional probability for $Y$ given $Z$.
% \end{proof}
\end{document}